\documentclass[11pt]{article}

\usepackage[final]{acl}

\usepackage{times}
\usepackage{latexsym}

\usepackage[T1]{fontenc}

\usepackage[utf8]{inputenc}

\usepackage{microtype}

\usepackage{inconsolata}

\usepackage{graphicx}

\usepackage{hyperref}
\usepackage{bbm}
\usepackage{booktabs}
\usepackage{multirow}
\usepackage{algorithm}
\usepackage{algorithmic}
\usepackage{enumitem}
\usepackage{enumerate}
\usepackage{graphicx}
\usepackage{subcaption}
\usepackage{diagbox} 
\usepackage{balance}
\usepackage[normalem]{ulem}
\useunder{\uline}{\ul}{}
\usepackage{pifont}
\usepackage[table,xcdraw]{xcolor}
\newcommand{\cmark}{{\color{green!82!black}\ding{51}}}

\newcommand{\xmark}{{\color{red!90!black}\ding{55}}}
\usepackage{listings}
\usepackage[most]{tcolorbox}
\tcbuselibrary{breakable}
\usepackage{cuted}  

\definecolor{codegreen}{rgb}{0,0.6,0}
\definecolor{codegray}{rgb}{0.5,0.5,0.5}
\definecolor{codepurple}{rgb}{0.58,0,0.82}
\definecolor{backorange}{RGB}{255,250,240}
\definecolor{frameorange}{RGB}{255,140,0}
\definecolor{codebg}{rgb}{0.95,0.95,0.95}

\usepackage{amsmath}
\usepackage{amssymb}
\usepackage{amsthm}

\newtheorem{theorem}{Theorem}

\theoremstyle{definition}

\theoremstyle{remark}

\usepackage{makecell}
\usepackage{array}

\newcommand{\eg}{\emph{e.g.}}
\newcommand{\ie}{\emph{i.e.}}

\lstdefinestyle{mystyle}{
    backgroundcolor=\color{backorange},   
    commentstyle=\color{codegreen},
    keywordstyle=\color{magenta},
    numberstyle=\tiny\color{codegray},
    stringstyle=\color{codepurple},
    basicstyle=\ttfamily\footnotesize,
    breakatwhitespace=false,         
    breaklines=true,                 
    captionpos=b,                    
    keepspaces=true,                 
    numbers=left,                    
    numbersep=5pt,                  
    showspaces=false,                
    showstringspaces=false,
    showtabs=false,                  
    tabsize=2
}
\newtcolorbox{pythonbox}[1][]{
  enhanced,
  breakable,
  title={Python Code Using COPT},
  colframe=frameorange,
  colback=backorange,
  coltitle=white,
  fonttitle=\bfseries,
  attach boxed title to top left={xshift=5mm, yshift*=-\tcboxedtitleheight/2},
  boxed title style={
    frame hidden, size=small, boxrule=0pt, arc=4pt, colback=frameorange
  },
  arc=3mm, boxrule=0.5pt, top=12pt,
  #1
}

\title{OSCAR: Optimization-Steered Agentic Planning for\\ Composed Image Retrieval}

\newcounter{equalcontrib}
\newcounter{corrauth}

\author{
  \textbf{Teng Wang\textsuperscript{1}\thanks{Equal Contribution}}
  \setcounter{equalcontrib}{\value{footnote}},
  \textbf{Rong Shan\textsuperscript{2,3}}\footnotemark[\value{equalcontrib}],
  \textbf{Jianghao Lin\textsuperscript{2}\thanks{Corresponding author}}
  \setcounter{corrauth}{\value{footnote}},
  \\
  \textbf{Junjie Wu\textsuperscript{1}},
  \textbf{Tianyi Xu\textsuperscript{2}},
  \textbf{Jianping Zhang\textsuperscript{2}},
  \textbf{Wenteng Chen\textsuperscript{2}},
  \\
  \textbf{Changwang Zhang\textsuperscript{1}},
  \textbf{Zhaoxiang Wang\textsuperscript{1}},
  \textbf{Weinan Zhang\textsuperscript{2}},
  \textbf{Jun Wang\textsuperscript{1}}\footnotemark[\value{corrauth}]
  \\
  \textsuperscript{1} OPPO,
  \textsuperscript{2} Shanghai Jiao Tong University,
  \textsuperscript{3} Shanghai Innovation Institute
  \\
  \texttt{wt0318@connect.hku.hk, linjianghao@sjtu.edu.cn, junwang.lu@gmail.com}
}

\begin{document}
\maketitle
\begin{abstract}
Composed image retrieval (CIR) requires complex reasoning over heterogeneous visual and textual constraints.
Existing approaches largely fall into two paradigms: \textit{unified embedding retrieval}, which suffers from \textit{single-model myopia}, and \textit{heuristic agentic retrieval}, which is limited by \textit{suboptimal, trial-and-error orchestration}.
To this end, we propose \textbf{OSCAR}, an \underline{o}ptimization-\underline{s}teered agentic planning framework for \underline{c}omposed im\underline{a}ge \underline{r}etrieval.
We are the first to reformulate agentic CIR from a heuristic search process into a principled trajectory optimization problem.
Instead of relying on heuristic trial-and-error exploration, OSCAR employs a novel offline-online paradigm.
In the offline phase, we model CIR via atomic retrieval selection and composition as a two-stage mixed-integer programming problem, mathematically deriving optimal trajectories that maximize ground-truth coverage for training samples via rigorous boolean set operations.
These trajectories are then stored in a golden library to serve as in-context demonstrations for online steering of VLM planner at online inference time. 
Extensive experiments on three public benchmarks and a private industrial benchmark show that OSCAR consistently outperforms SOTA baselines.
Notably, it achieves superior performance using only \textbf{10\%} of training data, demonstrating strong generalization of planning logic rather than dataset-specific memorization.
Our code is available\footnote{\url{https://github.com/tengwang0318/OSCAR}}. 
\end{abstract}

\section{Introduction}


\begin{figure}[t]
  \centering
  \includegraphics[width=0.48\textwidth]{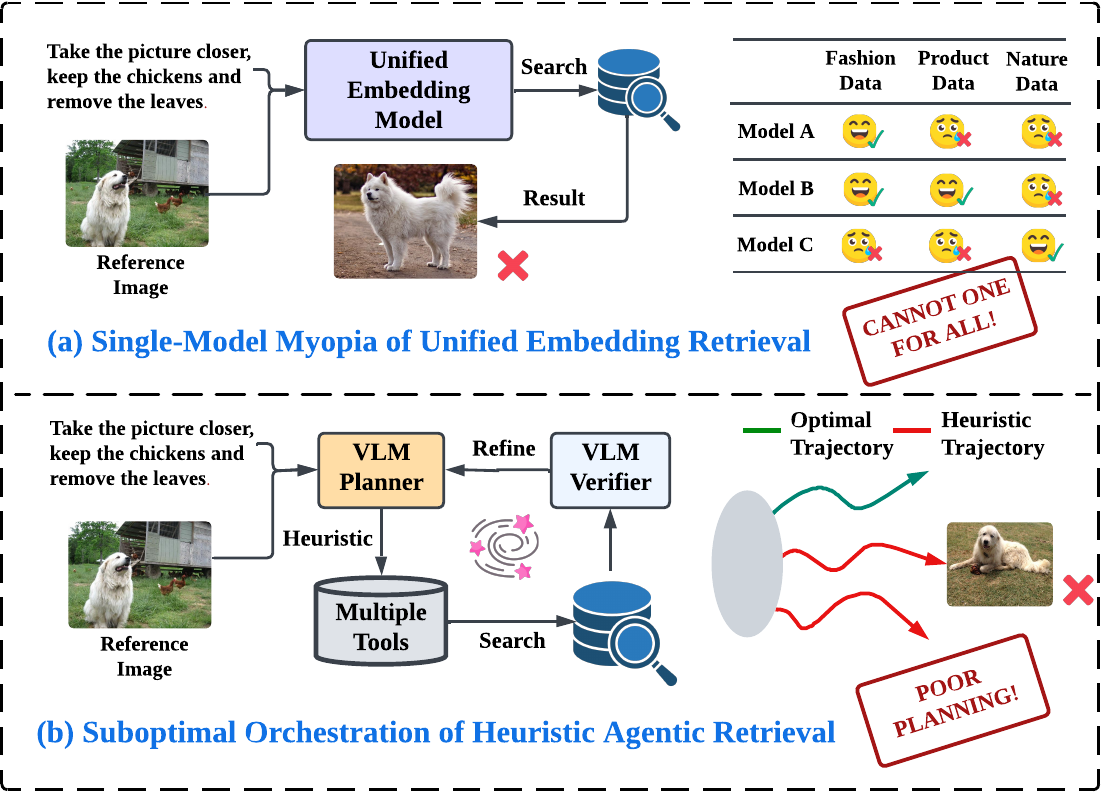}
  \vspace{-17pt}
  \caption{The illustration of limitations of existing image retrieval methods, \ie, (a) single-model myopia of unified embedding retrieval, and (b) suboptimal orchestration of heuristic agentic retrieval.
  }
  \vspace{-17pt}
  \label{fig:intro}
\end{figure}

With the evolution of modern retrieval systems, real-user queries have become increasingly complex and multimodal, which often involve compositional descriptions, multiple constraints, and implicit reasoning over both visual and textual information~\cite{cirsurvey,datta2008image,lin2024rella,shan2025automatic}. 
In response, a growing body of research has emerged to address composed image retrieval, generally falling into two paradigms: \textit{Unified Embedding Retrieval} and \textit{Heuristic Agentic Retrieval}. 
While both have yielded promising advances, they exhibit fundamental limitations that increasingly bottleneck real-world performance, as shown in Figure~\ref{fig:intro}.

\textit{Firstly}, \textbf{unified embedding retrieval} attempts to resolve complex queries via a single, monolithic representation model, whether using a general-purpose multimodal encoder~\cite{vlm2vec,QQMM,flowcir} or a domain-specific model fine-tuned for a particular setting~\cite{ldre, clip}. However, they suffer from \textbf{single-model myopia}. 
Concretely, this paradigm implicitly assumes that a \emph{single latent space} can act as a universal solution, whereas real-world queries are inherently heterogeneous. 
User intents vary drastically in granularity, ranging from high-level stylistic shifts (\eg, more formal) to fine-grained attribute constraints (\eg, specific color, texture, or pattern), and the optimal visual evidence shifts across domains~\cite{circo,cirr}. 
Under such diversity, relying on a single embedding space inevitably leads to short-sighted behaviors. 
As illustrated in Figure~\ref{fig:intro}(a), a model optimized for natural images may fail to resolve fashion-specific attributes, while a fashion-oriented model may struggle with natural scene reasoning. 
Consequently, even state-of-the-art single-model systems remain brittle in practice, as they are unable to adapt to intents that evolve or drift across diverse domains.

\textit{Secondly}, \textbf{heuristic agentic retrieval} seeks to overcome the representation bottleneck by decomposing the task into multi-step workflows invoked by Large Language Models (LLMs)~\cite{gpt4,qwen3,xi2026survey,zhou2026externalization,hrm} or Vision--Language Models (VLMs)~\cite{Qwen3-VL,colorbrowseragent}. These methods leverage external tools (\eg, captioners, rewriters, and retrievers) to handle complex queries~\cite{autocir,xr}. 
However, despite their flexibility, these approaches struggle with \textbf{suboptimal orchestration}, as shown in Figure~\ref{fig:intro}(b). 
Current pipelines rely heavily on the agent's internal heuristics (\eg, ReAct-style loops~\cite{react}), where the model makes greedy, iterative decisions on which tool to call next based on intermediate outputs~\cite{zhu2025evolutionary}. 
This lack of a global objective leads to unstructured and inefficient trajectories characterized by redundant calls, poor logical ordering, and unreliable handling of set-theoretic constraints (inclusion/exclusion). Furthermore, the reliance on iterative interaction incurs significant computational overhead without offering any guarantees towards the optimal retrieval path.


To this end, we propose \textbf{OSCAR}, an \underline{o}ptimization-\underline{s}teered agentic planning framework for \underline{c}omposed im\underline{a}ge \underline{r}etrieval, which, for the first time, reframes agentic CIR from a heuristic search process into a principled trajectory planning problem. 
Unlike prior methods that rely on heuristic, trial-and-error LLM exploration, OSCAR introduces a novel \textit{offline-online} paradigm.
During the offline stage, we treat each individual retriever invocation as the fundamental unit, and derive optimal tool-call planning trajectories for training samples by solving a two-stage mixed-integer programming (MIP) formulation.
These optimal trajectories are then stored in a \textit{golden library} to serve as reasoning demonstrations during inference.
At test time, OSCAR acquires relevant planning patterns to \textit{steer} the VLM agent, enabling it to replicate optimal logical reasoning for CIR without requiring expensive iterative search or human annotation.

In summary, our contributions are as follows:
\begin{itemize}[leftmargin=*]
    \item \textbf{Optimization Perspective.}
    To the best of our knowledge, we are the first to formulate agentic CIR as a MIP problem.
    By shifting from heuristic search to global optimization, we mathematically derive \textit{optimal} planning trajectories for training samples that maximize ground-truth coverage while minimizing computational redundancy, providing strong demonstration signals without human annotation.

    \item \textbf{Set-Theoretic Composition Logic}. We introduce a rigorous logic for composing CIR results via boolean set operations (\ie, union, intersection, and difference).
    This enables the agent to perform explicit inclusion and conservative exclusion reasoning, a capability that is mathematically intractable for single-embedding models and heuristic agentic methods.
    
    \item \textbf{The OSCAR Framework}.
    We propose a novel offline-online paradigm that bridges MIP optimization and agentic planning.
    By storing MIP solutions into a golden library as demonstrations, OSCAR effectively steers VLMs to perform complex compositional planning for CIR in a single inference pass.
    
    \item \textbf{Empirical Superiority}. 
    On three public benchmarks and a private industrial benchmark, OSCAR consistently outperforms SOTA single-embedding and agentic baselines. 
    Notably, it achieves these gains with only \textbf{10\%} of the training data for library construction, demonstrating strong generalization of abstract planning logic rather than simple memorization.
\end{itemize}

\begin{figure*}[t]
  \centering
  \includegraphics[width=0.95\textwidth]{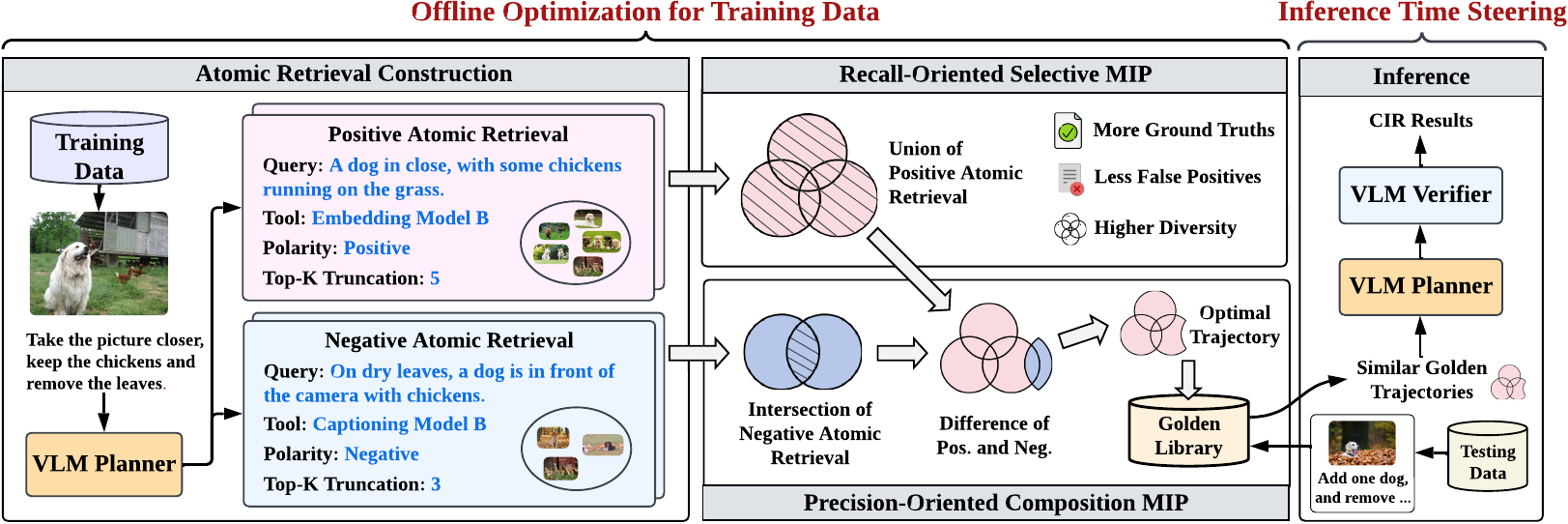}
    \vspace{-5pt}
    \caption{The overall framework of our proposed OSCAR. 
}
    \vspace{-10pt}
\label{fig:framework}
\end{figure*}
\section{Related Works}
\label{sec:related_work}

Composed image retrieval (CIR) aims to retrieve target images that satisfy specific textual modifications while preserving the visual context of a reference image.
This task demands fine-grained compositional reasoning over multimodal heterogeneous attributes with increasing semantic complexity~\cite{cirsurvey,circo}.

\textbf{Unified Embedding Retrieval for CIR}.
A dominant line of work formulates CIR by collapsing a composed query, which consists of a reference image and a modification text, into a single embedding, followed by nearest-neighbor retrieval.
Composition is typically achieved via cross-modal fusion, feature editing, or language-aligned representations in a shared latent space~\cite{clip4cir,pic2word}.
Moreover, recent large-scale benchmarks~\cite{vlm2vec} have driven rapid progress in multimodal embedding models for CIR, with a series of methods achieving increasingly strong performance~\cite{QQMM}. 
However, different models exhibit varying strengths across different domains and datasets, rendering the single-model myopia problem.

\textbf{Heuristic Agentic Retrieval for CIR.}
Recent agentic methods decompose CIR into multi-step workflows, such as query rewriting, retrieval, and correction~\cite{autocir}.
However, they are typically driven by heuristic exploration, often requiring repeated model interactions, incurring high computational cost, and lacking global optimality guarantees.
In contrast, our proposed OSCAR replaces this heuristic trial-and-error with optimization-steered planning, using an offline MIP formulation to guide the agent towards provably efficient and accurate retrieval trajectories.

More detailed related work is provided in Appendix~\ref{sec:related_work_full}.

\section{Methodology}
\label{sec:method}

In this section, we present \textbf{OSCAR}, an \underline{O}ptimization-\underline{S}teered Agentic Planning Framework for \underline{C}omposed Im\underline{a}ge \underline{R}etrieval. 
The overall framework of OSCAR is shown in Figure~\ref{fig:framework}.
All the prompts for VLM invocation are provided in Appendix~\ref{appendix:prompts}.

\subsection{Preliminaries}

\vspace{3pt}
\textbf{Problem Formulation}.
Let $q=\{q_{img}, q_{txt}\}$ denote a composed query, where $q_{img}$ is the reference image and $q_{txt}$ is the modification text.
The goal of CIR is to retrieve a set of ground-truth images $I^{+} \subset \mathcal{I}$ from an image gallery $\mathcal{I}$ based on the composed query $q$.

Standard approaches typically learn a monolithic retrieval model to perform CIR in a unified latent space, thereby suffering from single-model myopia. 
In response, agentic methods reframe these existing retrieval models as a set of \textbf{atomic tools}, denoted as $\mathcal{T} = \{f_1, f_2, ..., f_m\}$. 
Each tool accepts the composed query input and returns a set of candidate images. 
This formulation allows us to treat the retrieval process not as a single inference step, but as a \textit{compositional planning problem} over the space of available tools.

\vspace{5pt}
\noindent\textbf{From Heuristic Exploration to Global Optimization}.
Existing agentic methods typically navigate this tool space via \textit{heuristic exploration} (\eg, ReAct loops~\cite{react}), where an LLM/VLM iteratively decides the next step based on local context.
While flexible, they lack a global view of the solution space: agents may redundantly query overlapping concepts or fail to enforce strict exclusion constraints.
We observe that for any given query $q$ with \textit{known} ground truth $\mathcal{I}^{+}$, the ideal tool usage can be mathematically formulated as a \textbf{Set Cover}~\cite{setcover} variation: 
we seek a compact subset of atomic tool outputs whose union covers $\mathcal{I}^{+}$ while minimizing the overlap between this union and non-targets $\mathcal{I}^{-}=\mathcal{I}\setminus \mathcal{I}^+$.
This realization allows us to reframe the planning process not as a trial-and-error search, but as a MIP problem that explicitly solves for the global optimal trajectory.

\vspace{5pt}
\noindent\textbf{The OSCAR Framework}.
As illustrated in Figure~\ref{fig:framework}, since the ground-truth set is \textit{unknown} during the retrieval, OSCAR bridges the gap via a distinct \textit{offline-online} paradigm that steers the retrieval process using golden demonstrations derived from training data.

\begin{itemize}[leftmargin=10pt]
    \item \textbf{Offline Optimization}. 
    For each training sample, we first construct and perform \textit{atomic retrievals} by varying the query and top-$k$ truncation.
    Then, we optimize the global planning trajectory via two-stage MIPs: (1) a \textit{recall-oriented} MIP that identifies the optimal subset of atomic retrievals to maximize ground-truth coverage, and (2) a \textit{precision-oriented} MIP that combines rigorous set-theoretic operations (union, intersection, difference) to filter irrelevant images.
    The resulting optimal trajectories are stored in a \textit{golden library} to steer the online agentic planning for CIR.
    
    \item \textbf{Online Steering}. For a test query with unknown ground truth, we select optimal trajectories associated with similar queries from the golden library as in-context demonstrations.
    In this way, we effectively generalize the high-level logic derived offline and steer the agent toward near-optimal planning for CIR.
\end{itemize}

\subsection{Atomic Retrieval Construction}
\label{sec:method:candidate_query_generation}

To enable fine-grained planning, OSCAR first decomposes the complex user intent into a search space of discrete, executable actions.
We define an \textbf{atomic retrieval} $r$ as the fundamental unit of this space, formalized as a four-tuple:
\begin{equation}
    r = (f,\; \hat{q},\; p,\; k),
\end{equation}
where $f \in \mathcal{T}$ is the retrieval tool (\eg, a caption-based searcher or a multimodal encoder), $\hat{q}=\{q_{img},\hat{q}_{txt}\}$ is derived by rewriting the textual query, $p \in \{+, -\}$ indicates the polarity (positive for inclusion, negative for exclusion), and $k$ is the number of candidate images to return (\ie, top-$k$ truncation).
Each atomic retrieval $r$ represents an indivisible planning unit that returns a specific candidate image set $\mathcal{S}_r \subset \mathcal{I}$.

\textbf{Query Decomposition and Polarity}.
Rather than solely relying on the original query, we employ a VLM to generate a diverse set of rewritten queries.
The VLM is prompted to decompose the modification instructions into specific visual attributes or semantic constraints.
Crucially, each generated query $\hat{q}$ is assigned a polarity $p$. 
As illustrated in Figure~\ref{fig:framework}, positive queries ($p=+$) target attributes that must be present in the result (\eg, chickens), while negative queries ($p=-$) target attributes that must be explicitly excluded (\eg, leaves).
This decoupling allows the subsequent optimization stage to leverage boolean set logic for precise inclusion and exclusion, a capability that a single retrieval model lacks.

\textbf{Top-$k$ Truncation}.
The truncation parameter $k$ dictates the trade-off between recall and precision for a given tool $f$.
To enable the planning over different granularities of scope, we discretize the truncation parameter $k$ into a finite set of levels:
\begin{equation}
    \mathcal{K} = \{k_1, k_2, \dots, k_{max}\}.
\end{equation}
For a fixed tool $f$ and query $\hat{q}$, the retrieved results are deterministic and monotonic w.r.t. $k$ (\ie, $\mathcal{S}_{k_1} \subset \mathcal{S}_{k_2}$ for $k_1 < k_2$).
To ensure computational efficiency, we execute the retrieval once with the maximum truncation $k_{\max}$ and generate the smaller top-$k$ variants via slicing, avoiding redundant model inferences.

By taking the Cartesian product over available tools $f$, rewritten queries $\hat{q}$, polarities $p$, and truncations $k$, we construct a diverse global set of atomic retrievals, denoted as $\mathcal{R}$ (1,182 atomic retrievals per sample).
This set serves as the decision space for the subsequent offline optimization stages, where we select and compose the optimal subset of $\mathcal{R}$ to satisfy the global retrieval objective.

\subsection{Recall-Oriented Selection MIP}
\label{sec:method:stage1}

The first-stage recall-oriented MIP aims to identify a compact subset of \textit{positive} atomic retrievals $\mathcal{R}^{+}$ that maximizes the coverage of ground-truth images $\mathcal{I}^{+}$ while minimizing the inclusion of irrelevant noise $\mathcal{I}^{-}$.
This selection process effectively prunes the massive atomic retrieval space constructed in Section~\ref{sec:method:candidate_query_generation}, resulting in a recall-oriented candidate set of images $\mathcal{U}\subset \mathcal{I}$ for the subsequent fine-grained composition stage.

\vspace{5pt}
\noindent\textbf{Optimization Variables}.
We define binary decision variables $x_r \in \{0, 1\}$ to indicate whether a positive atomic retrieval $r \in \mathcal{R}^+$ is selected.
We further define auxiliary state variables, which are logically determined by the decision variable $x_r$, to track image coverage and tool diversity:
\begin{itemize}[leftmargin=10pt]
    \item $c_i \in \{0, 1\}$ indicates whether the image $i \in \mathcal{I}$ is covered by \textit{at least one} selected atomic retrieval.
    \item $t_f \in \{0, 1\}$ indicates whether the tool $f$ is used by \textit{at least one} selected atomic retrieval, \ie, active.
\end{itemize}
Moreover, to prevent tool redundancy, we group atomic retrievals into different \textit{families} $\mathcal{F}$. A family $F \in \mathcal{F}$ consists of retrievals that share the same tool $f$, query $\hat{q}$, and polarity $p$, differing only in their truncation threshold $k$.
Since larger $k$ values strictly subsume smaller ones (monotonicity), selecting multiple members from the same family provides diminishing returns.

\vspace{5pt}
\noindent\textbf{MIP Formulation}.
We formulate the selection problem as a MIP that balances three competing objectives: (1) maximizing the ground-truth coverage $\sum_{i\in\mathcal{I}^+} c_i / |\mathcal{I}^{+}|$, (2) minimizing the irrelevant noise $\sum_{i\in\mathcal{I}^{-}} c_i / |\mathcal{I}^{-}|$, and (3) encouraging the tool diversity $\sum_{f\in \mathcal{T}}t_f$.
The optimization formulation is written as:
\begin{align}
    \max_{\{x_r\}_{r\in\mathcal{R^{+}}}} \quad & \frac{w_R}{|\mathcal{I}^+|} \sum_{i\in\mathcal{I}^+} c_i - \frac{w_P}{|\mathcal{I}^-|} \sum_{i\in\mathcal{I}^-} c_i \notag \\ & \quad + \lambda_{\mathrm{div}} \sum_{f\in\mathcal{T}} t_f \label{eq:method stage1_obj}\\
    \text{s.t.} \quad & \sum_{r\in F} x_r \le 1, \quad \forall F\in\mathcal{F}, \label{eq:method stage1_family}\\
    & x_r,\, c_i,\, t_f \in \{0,1\}. \label{eq:method stage1_bin}
\end{align}
The objective function in Eq.~\ref{eq:method stage1_obj} prioritizes high recall ($w_R$) to ensure the ground truth is present in the candidate universe, while imposing a penalty ($w_P$) on expanding the search space with irrelevant images.
The term $\lambda_{\mathrm{div}}$ serves as a regularizer to prevent over-reliance on a single tool, \ie, single-model myopia.
The constraint in Eq.~\ref{eq:method stage1_family} enforces \textit{truncation exclusivity}: for any given tool and query configuration, the solver must pick at most one optimal $k$ threshold, preventing redundant coverage of nested subsets.

\vspace{5pt}
\noindent\textbf{Optimization Output}.
The solution to this MIP yields the optimal set of positive atomic retrievals, denoted as $\mathcal{R}^{+}_{*}$, which produces the recall-oriented candidate set of images $\mathcal{U} = \bigcup_{r \in \mathcal{R}^{+}_{*}} \mathcal{S}_r$.
While $\mathcal{U}$ ensures high coverage of the ground truth, it inevitably contains irrelevant images since we omit the negative atomic retrieval.
Hence, the next stage performs precision-oriented filtering to remove the remaining noise.

\subsection{Precision-Oriented Composition MIP}
\label{sec:method:stage2}

While the first stage ensures high recall, the resulting candidate image set $\mathcal{U}$ inevitably contains significant noise since the negative atomic retrievals are not considered yet.
Hence, we perform logical filtering and refine $\mathcal{U}$ by composing atomic retrievals into a rigorous boolean expression.
To ensure stability and interpretability, we restrict the search space to a fixed two-clause structure:
\begin{equation}
    \mathcal{S}_{final} \;=\;
    \underbrace{\Big(\bigcup_{r \in \mathcal{R}^{+}_{**}} \mathcal{S}_r\Big)}_{\text{Positive Union}}
    \;\setminus\;
    \underbrace{\Big(\bigcap_{r \in \mathcal{R}^{-}_{*}} \mathcal{S}_r\Big)}_{\text{Negative Intersection}}
    \label{eq:method stage2_structure}
\end{equation}
where $\mathcal{R}^{+}_{**}\subseteq \mathcal{R}^{+}_{*} \subseteq \mathcal{R}^+$ and $\mathcal{R}^{-}_{*} \subseteq \mathcal{R}^-$ denote the optimal subsets of positive and negative atomic retrievals selected by this stage.
$\mathcal{R}^{+}_{*}$ is the optimal solution from the previous stage.
Crucially, we employ \textbf{Intersection} for the negative clause to enforce \textit{conservative exclusion}: an image is removed only if \textit{all} selected negative tools agree it is irrelevant, thereby preventing accidental deletion of ground-truth samples due to single-tool hallucinations.
We further compare this conservative design with more aggressive negative
filtering strategies in Appendix~\ref{appendix:negative_filtering}.

\vspace{5pt}
\noindent\textbf{Optimization Variables}.
We define binary decision variables $x_r \in \{0, 1\}$ to indicate whether an atomic retrieval $r\in\mathcal{R}_{*}^{+} \cup \mathcal{R}^{-}$ is selected.
Based on Eq.~\ref{eq:method stage2_structure}, these decisions determine the state of each candidate image $i \in \mathcal{U}$ via three logically coupled variables:
\begin{itemize}[leftmargin=10pt]
    \item $u_i \in \{0, 1\}$ indicates if image $i$ is covered by the \textit{Positive Union}, \ie, retrieved by \textit{at least one} selected positive atomic retrieval.
    \item $v_i \in \{0, 1\}$ indicates if image $i$ falls into the \textit{Negative Intersection}, \ie, retrieved by \textit{all} selected negative atomic retrievals.
    \item $z_i \in \{0, 1\}$ indicates if image $i$ remains in the final set $\mathcal{S}_{final}$. Logically, $z_i = 1$ if and only if $u_i=1$ AND $v_i=0$ (set difference).
\end{itemize}

\vspace{5pt}
\noindent\textbf{MIP Formulation}.
The objective is to maximize the retention of ground-truth images in the final set $\mathcal{S}_{final}$. 
We introduce a regularization term $\lambda_{reg}$ to mitigate the false positives and thereby prevent trivial solutions, \eg, selecting no negative tools to maximize recall.
The MIP is formulated as follows:
\begin{align}
    \max_{\{x_r\}_{r\in \mathcal{R^{+}_{*}}\cup \mathcal{R}^{-}}} \quad & \sum_{i \in \mathcal{U} \cap \mathcal{I}^+} z_i - \lambda_{reg} \sum_{i \in \mathcal{I}^-} z_i
    \label{eq:method stage2_obj}\\
    \text{s.t.}\quad
    & \sum_{r \in \mathcal{R}^+_{*}} x_r \ge 1, \label{eq:method stage2_nonempty}\\
    & x_r, z_i \in \{0,1\}. \label{eq:method stage2_bin}
\end{align}
The constraint in Eq.~\ref{eq:method stage2_nonempty} ensures that at least one positive atomic retrieval is selected, enforcing the positive union to be non-empty.

\vspace{5pt}
\noindent\textbf{Optimization Output}.
The solution yields an optimal plan $(\mathcal{R}^{+}_{**}, \mathcal{R}^{-}_{*})$ consisting of the selected positive and negative atomic retrievals.
Together, they define an optimal trajectory of agentic composed image retrieval for the training sample, \ie, a sequence of specific tool calls associated with the rewritten queries and top-$k$ truncations, and the corresponding set operations leading to the ground truth.
These trajectories are stored in a \textit{Golden Library} for online steering during test-time inference for agentic composed image retrieval.
We discuss the design choice of the structured composition template and its comparison with a more expressive DNF formulation in Appendix~\ref{appendix:rationale_dnf}.

\subsection{Optimization-Steered Inference for CIR}
\label{sec:method:case_lib}

The optimization pipeline described in Sections~\ref{sec:method:stage1} and \ref{sec:method:stage2} is computationally intensive and relies on ground-truth availability, making it inapplicable during inference.
However, the trajectories produced by these stages represent \textit{golden demonstrations} of ideal planning logic.
To transfer this optimized reasoning to the CIR agent, we construct a \textbf{Golden Library} that turns offline insights into retrieval-augmented in-context demonstrations.

For each training instance, we employ \texttt{Qwen3-Embedding-8B} to encode the problem context, which is the concatenation of the modification query $q_{txt}$ and the \textit{caption} of reference image $q_{img}$. 
The caption is produced by \texttt{Qwen3-VL-32B}.
The resulting problem-context vector serves as the key for indexing, and the value is the corresponding two-stage MIP solution.
At test time, as shown in Figure~\ref{fig:framework}, given a composed query, we first compute its problem-context embedding and then retrieve the top-$N$ most similar cases from the golden library by cosine similarity.
These trajectories serve as the in-context demonstrations for the VLM planner, whose output is fed to the VLM verifier for final ranking.

Crucially, this process does not merely provide the answers to similar questions, but rather transfers \textbf{generalized planning logic}.
By observing how the optimal planner handled similar semantic structures, the planner learns to replicate the underlying reasoning strategy, \eg, selecting the appropriate tool types, determining the correct polarity for exclusion, and calibrating the truncation threshold $k$.
The robustness of this logic transfer is evidenced by our experiments: OSCAR achieves state-of-the-art performance even when the golden library is constructed using only \textbf{10\%} of the available training data, demonstrating that the system learns abstract meta-strategies rather than memorizing dataset-specific solutions.


We briefly discuss key design choices of OSCAR, with full implementation details and rigorous formulations provided in Appendix~\ref{appendix:method_discussion}.

\section{Experiment}
\label{sec:experiment}

\subsection{Experiment Setups}

\paragraph{Datasets and Metrics.}
Our main experiments are conducted on three representative public datasets for CIR: \textit{CIRCO}~\cite{circo}, \textit{CIRR}~\cite{cirr} and \textit{FashionIQ}~\cite{fashioniq}. We report Recall@$K$ on CIRR and FashionIQ, and mAP@$K$ on CIRCO since each CIRCO query has multiple ground-truth images, following standard practice~\cite{fire, autocir}. 
Detailed dataset statistics and splits are provided in Appendix~\ref{appendix:dataset_stats}.
\paragraph{Baselines.}
We compare OSCAR with four groups of baselines:
\begin{itemize}[leftmargin=10pt, topsep=1pt, itemsep=0pt, parsep=0pt]
    \item \textbf{Multimodal embedding models:} RzenEmbed-v2-7B~\cite{rzenemb}, VLM2Vec~\cite{vlm2vec}, B3-Qwen2-7B~\cite{B3},  Ops-MM-embedding-v1-7B~\cite{ops}, and QQMM-embed-v2~\cite{QQMM}.
    
    \item \textbf{Caption-based text embedding models:} Qwen3-VL-32B~\cite{Qwen3-VL} for captioning, with bge-m3~\cite{bge-m3} and Qwen3-Embedding (0.6B, 4B, 8B)~\cite{qwen3_embedding} for retrieval.
    
    \item \textbf{CIR-dedicated methods:} Pic2Word~\cite{pic2word}, SEARLE~\cite{circo}, SEARLE-XL-OTI~\cite{circo}, CIReVL~\cite{cirevl}, LinCIR~\cite{lincir}, LDRE~\cite{ldre}, and FiRE~\cite{fire}.
    
    \item \textbf{Agentic CIR methods:} MRA-CIR~\cite{MRA-CIR}, AutoCIR~\cite{autocir}, and $X^R$~\cite{xr}. OSCAR also belongs to this category.
\end{itemize}

\paragraph{Implementation Details}
\label{sec:exp:implementation}
Due to the page limitation, we provide the implementation details in Appendix~\ref{appendix:exp implementation}.

\begin{table*}[t]
\centering

\vspace{-8pt}
\resizebox{0.9\textwidth}{!}{
\renewcommand\arraystretch{1.03}
\begin{tabular}{ccccccc|cccc}
\toprule
\hline
\multicolumn{1}{c|}{} &                  & \multicolumn{1}{c|}{}                      & \multicolumn{4}{c|}{CIRCO}    & \multicolumn{4}{c}{CIRR}                           \\
\multicolumn{1}{c|}{\multirow{-2}{*}{Type}} &
  \multirow{-2}{*}{Method} &
  \multicolumn{1}{c|}{\multirow{-2}{*}{\begin{tabular}[c]{@{}c@{}}No Parameter\\ Update?\end{tabular}}} &
  m@5 &
  m@10 &
  m@25 &
  m@50 &
  R@1 &
  R@5 &
  R@10 &
  R@50 \\ \hline
\multicolumn{1}{c|}{} & Ops-MM-v1-7B     & \multicolumn{1}{c|}{\cmark} & 13.56 & 15.96 & 18.61 & 19.68 & 1.90        & 50.29 & 66.68       & 91.64          \\
\multicolumn{1}{c|}{} & RzenEmbed-v2-7B  & \multicolumn{1}{c|}{\cmark} & 32.20 & 34.19 & 37.41 & 38.61 & 19.30       & 70.36 & 83.08       & \textbf{96.77} \\
\multicolumn{1}{c|}{} & VLM2Vec          & \multicolumn{1}{c|}{\cmark} & 3.40  & 4.07  & 5.08  & 5.74  & 0.10        & 26.48 & 41.35       & 71.74          \\
\multicolumn{1}{c|}{} & B3\_Qwen2\_7B    & \multicolumn{1}{c|}{\cmark} & 3.67  & 4.55  & 5.53  & 6.13  & 0.80        & 37.95 & 54.39       & 81.01          \\
\multicolumn{1}{c|}{\multirow{-5}{*}{\begin{tabular}[c]{@{}c@{}}Multimodal\\ Embedding\end{tabular}}} &
  QQMM-embed-v2 &
  \multicolumn{1}{c|}{\cmark} &
  {\ul 45.92} &
  {\ul 47.13} &
  {\ul 50.39} &
  {\ul 51.45} &
  28.98 &
  73.66 &
  82.98 &
  {\ul 96.72} \\ \hline
\multicolumn{1}{c|}{} & bge-m3           & \multicolumn{1}{c|}{\cmark} & 8.15  & 8.61  & 9.62  & 10.21 & 12.53       & 34.00 & 48.63       & 75.06          \\
\multicolumn{1}{c|}{} & Qwen3-Embed-0.6B & \multicolumn{1}{c|}{\cmark} & 9.55  & 10.35 & 11.61 & 12.33 & 14.87       & 39.11 & 54.53       & 82.51          \\
\multicolumn{1}{c|}{} & Qwen3-Embed-4B   & \multicolumn{1}{c|}{\cmark} & 13.55 & 14.60 & 16.33 & 17.20 & 22.17       & 51.57 & 64.84       & 88.19          \\
\multicolumn{1}{c|}{\multirow{-4}{*}{\begin{tabular}[c]{@{}c@{}}Text\\ Embedding\end{tabular}}} &
  Qwen3-Embed-8B &
  \multicolumn{1}{c|}{\cmark} &
  16.45 &
  17.23 &
  18.98 &
  19.88 &
  23.21 &
  52.48 &
  66.51 &
  89.71 \\ \hline
\multicolumn{1}{c|}{} & Pic2Word         & \multicolumn{1}{c|}{\xmark} & 8.72  & 9.51  & 10.46 & 11.29 & 23.90       & 51.70 & 65.30       & 87.80          \\
\multicolumn{1}{c|}{} & SEARLE           & \multicolumn{1}{c|}{\xmark} & 11.68 & 12.73 & 14.33 & 15.12 & 24.24       & 52.48 & 66.29       & 88.84          \\
\multicolumn{1}{c|}{} & SEARLE-XL-OTI    & \multicolumn{1}{c|}{\xmark} & 10.18 & 11.03 & 12.72 & 13.67 & 24.87       & 52.31 & 66.29       & 88.58          \\
\multicolumn{1}{c|}{} & CIReVL           & \multicolumn{1}{c|}{\cmark} & 18.57 & 19.01 & 20.89 & 21.80 & 24.55       & 52.31 & 64.92       & 86.34          \\
\multicolumn{1}{c|}{} & LinCIR           & \multicolumn{1}{c|}{\xmark} & 12.59 & 13.58 & 15.00 & 15.85 & 25.04       & 53.25 & 66.68       & -              \\
\multicolumn{1}{c|}{} & LDRE             & \multicolumn{1}{c|}{\cmark} & 23.35 & 24.03 & 26.44 & 27.50 & 26.53       & 55.57 & 67.54       & 88.50          \\
\multicolumn{1}{c|}{\multirow{-7}{*}{\begin{tabular}[c]{@{}c@{}}CIR\\ Dedicated\end{tabular}}} &
  FiRE &
  \multicolumn{1}{c|}{\xmark} &
  31.03 &
  32.08 &
  34.40 &
  35.50 &
  43.33 &
  {\ul 74.02} &
  83.51 &
  95.83 \\ \hline
\multicolumn{1}{c|}{} & MRA-CIR          & \multicolumn{1}{c|}{\xmark} & 27.14 & 28.85 & 31.54 & 32.63 & 37.98       & 67.45 & 78.07       & 93.98          \\
\multicolumn{1}{c|}{} & AutoCIR          & \multicolumn{1}{c|}{\cmark} & 24.05 & 25.14 & 27.35 & 28.36 & 31.81       & 61.95 & 73.86       & 92.07          \\
\multicolumn{1}{c|}{} & $X^R$            & \multicolumn{1}{c|}{\cmark} & 31.38 & 32.88 & 35.46 & 36.50 & {\ul 43.13} & 73.59 & {\ul 83.09} & 94.05          \\
\multicolumn{1}{c|}{\multirow{-4}{*}{Agentic}} &
  \textbf{OSCAR (Ours)} &
  \multicolumn{1}{c|}{\cmark} &
  \textbf{56.54} &
  \textbf{58.53} &
  \textbf{61.92} &
  \textbf{62.67} &
  \textbf{51.18} &
  \textbf{79.50} &
  \textbf{87.45} &
  96.56 \\ \hline
\rowcolor{purple!15}
\multicolumn{3}{c|}{\cellcolor{purple!15}Relative Improvement (\%)}                 &23.13\% & 24.19\% & 22.88\% & 21.81\% & 18.67\%       & 7.40\%  & 5.25\%        & -0.22\%          \\ \hline
   \bottomrule
\end{tabular}}
\caption{Performance comparison on \textit{CIRCO} and \textit{CIRR} datasets. 
m@$K$ and R@$K$ denote mAP@$K$ and Recall@$K$, respectively. 
Best results are highlighted in \textbf{bold}, while the second best are underlined. Relative improvement is calculated against the best baseline result. 
``-'' means that the result is not reported in the original paper. }
\label{tab:circo_cirr_main_table}

\end{table*}

\begin{table}[ht]
\centering

\vspace{-8pt}

\resizebox{0.49\textwidth}{!}{
\renewcommand\arraystretch{1.1}
\begin{tabular}{c|cc|cc|cc|cc}
\toprule
\hline
\multirow{2}{*}{Method} & \multicolumn{2}{c|}{Dress}      & \multicolumn{2}{c|}{Shirt}      & \multicolumn{2}{c|}{Toptee}     & \multicolumn{2}{c}{Average}     \\
 &
  \multicolumn{1}{c}{R@10} &
  \multicolumn{1}{c|}{R@50} &
  \multicolumn{1}{c}{R@10} &
  \multicolumn{1}{c|}{R@50} &
  \multicolumn{1}{c}{R@10} &
  \multicolumn{1}{c|}{R@50} &
  \multicolumn{1}{c}{R@10} &
  \multicolumn{1}{c}{R@50} \\ \hline
Ops-MM-v1-7B            & 19.39          & 38.13          & 31.45          & 48.87          & 27.69          & 46.51          & 26.18          & 44.50          \\
RzenEmbed-v2-7B         & {\ul 37.38}    & {\ul61.97}          & {\ul 45.63}    & 64.97          & {\ul 46.56}    & 67.57          & {\ul 43.19}    & {\ul 64.84}    \\
VLM2Vec                 & 4.76           & 14.77          & 15.60          & 30.62          & 11.37          & 22.95          & 10.58          & 22.78          \\
B3\_Qwen2\_7B           & 8.53           & 22.31          & 19.53          & 34.49          & 14.99          & 29.88          & 14.35          & 28.89          \\
QQMM-embed-v2           & 36.44          & 60.09          & \textbf{46.07} & {\ul 65.65}          & 46.35          & {\ul 68.49}    & 42.95          & 64.74          \\ \hline
bge-m3                  & 10.81          & 23.75          & 20.66          & 33.66          & 15.96          & 27.33          & 15.81          & 28.25          \\
Qwen3-Embed-0.6B        & 9.32           & 19.83          & 21.00          & 34.69          & 15.96          & 27.84          & 15.43          & 27.45          \\
Qwen3-Embed-4B          & 12.69          & 27.02          & 26.69          & 40.68          & 21.88          & 36.05          & 20.42          & 34.58          \\
Qwen3-Embed-8B         & 12.64          & 29.20           & 28.70          & 43.13          & 23.41          & 36.97          & 21.58          & 36.43          \\ \hline
Pic2Word                & 20.20          & 40.20          & 26.20          & 43.60          & 27.90          & 47.40          & 24.77          & 43.73          \\
SEARLE                  & 20.48          & 43.13          & 26.89          & 45.58          & 29.32          & 49.97          & 25.56          & 46.23          \\
SEARLE-XL-OTI           & 21.57          & 44.47          & 30.37          & 47.49          & 30.90          & 51.76          & 27.61          & 47.91          \\
CIReVL                  & 24.79          & 44.76          & 29.49          & 47.40          & 31.36          & 53.65          & 28.55          & 48.60          \\
LinCIR                  & 20.92          & 42.44          & 29.10          & 46.81          & 28.81          & 50.18          & 26.28          & 46.48          \\
LDRE                    & 22.93          & 46.76          & 31.04          & 51.22          & 31.57          & 53.64          & 28.51          & 50.54          \\
FiRE                    & 29.60          & 50.87          & 39.84          & 60.06          & 35.64          & 57.83          & 35.03          & 56.25          \\ \hline
MRA-CIR &
  \multicolumn{1}{c}{31.87} &
  \multicolumn{1}{c|}{54.23} &
  \multicolumn{1}{c}{40.43} &
  \multicolumn{1}{c|}{60.20} &
  \multicolumn{1}{c}{41.25} &
  \multicolumn{1}{c|}{62.51} &
  \multicolumn{1}{c}{37.85} &
  \multicolumn{1}{c}{58.98} \\
AutoCIR                 & 24.94          & 45.81          & 34.00          & 53.43          & 33.10          & 55.58          & 30.68          & 51.61          \\
$X^R$ &
  \multicolumn{1}{c}{28.71} &
  \multicolumn{1}{c|}{52.50} &
  \multicolumn{1}{c}{38.91} &
  \multicolumn{1}{c|}{56.82} &
  \multicolumn{1}{c}{43.91} &
  \multicolumn{1}{c|}{62.57} &
  \multicolumn{1}{c}{37.18} &
  \multicolumn{1}{c}{57.30} \\
\textbf{OSCAR (Ours)}    & \textbf{38.47} & \textbf{65.15} & 44.50          & \textbf{67.52} & \textbf{48.24} & \textbf{71.24} & \textbf{43.73} & \textbf{67.97} \\ 
\hline
\rowcolor{purple!15}
\multicolumn{1}{c|}{\cellcolor{purple!15}Rel. Improv. (\%)}                 &2.92\% & 5.13\% & -3.40\% & 2.84\% & 3.61\%       & 4.02\%  & 1.25\%       & 4.82\%          \\ \hline

\hline
   \bottomrule

\end{tabular}}

\caption{Performance comparison on \textit{FashionIQ} dataset.}
\label{tab:fashioniq_main_table}

\end{table}

\subsection{Main Results}
Table~\ref{tab:circo_cirr_main_table} and Table~\ref{tab:fashioniq_main_table} report the main results on standard CIR benchmarks. 
We also provide Recall@$k$ performance between the baselines and our methods on CIRCO dataset in Appendix ~\ref {appendix:circo_recall}. 
From the tables we can have the following observations:

\begin{table}[t]
\centering

\vspace{-8pt}
\resizebox{0.48\textwidth}{!}{
\renewcommand\arraystretch{1.2}
\begin{tabular}{c|ccccccc|cccccc}
\toprule
\hline
                           & \multicolumn{2}{c}{CIRCO} & \multicolumn{2}{c}{CIRR} & \multicolumn{2}{c}{FIQ.Avg} \\
\multirow{-2}{*}{Variants} & m@25         & m@50        & R@10         & R@50        & R@10      & R@50            \\ \hline
OSCAR (Ours)                    & \textbf{61.92} & \textbf{62.67} & \textbf{87.45}               & \textbf{96.56}               & 43.73          & \textbf{67.97} \\ 
w/o Demo.               & {\ul 59.72}    & {\ul 59.84}    & {75.01} & {82.02} & {\ul 46.57}    & 61.62          \\ 
w/o Set.Diff.             & 59.54        & 59.62       &  {\ul 87.32}           &   {\ul 93.80}          & 46.46     & {\ul 62.08}     \\ 
w/o Set.Diff. \& Demo. & 58.63          & 58.66          & 39.11                        & { 54.53} & \textbf{49.44} & 53.96          \\ \hline
   \bottomrule
\end{tabular}}
\caption{Ablation study on different variants of OSCAR. \textit{FIQ.Avg} denotes average results over FashionIQ subsets. }
\label{tab:ablation_all_dataset}

\end{table}

\begin{itemize}[leftmargin=10pt]
    \item \textbf{Overall Performance.} Generally, OSCAR achieves the best performance across all datasets, consistently outperforming other baseline methods, with demonstrations derived from only \textbf{10\%} training data. Notably, OSCAR is a training-free framework based on open-sourced models and requires no parameter update, even surpassing other baselines that are based on domain-specific-tuning or close-sourced LLMs.

    \item \textbf{Comparison with single-embedding methods.}
    OSCAR consistently outperforms both multimodal and text embedding baselines, showing that optimization-steered planning and structured tool composition are more effective than relying on a single unified representation.


\item \textbf{Comparison with CIR-dedicated methods.}
OSCAR outperforms CIR-dedicated methods, while our framework does not rely on specialized fusion, rewriting, or learned transformation modules.
This suggests that explicit planning over general-purpose tools can be a reusable alternative to dataset-specific engineering.

 \item \textbf{Comparison with agentic approaches.}
OSCAR consistently outperforms prior agentic retrieval methods, while avoiding iterative agent interaction or multi-round reasoning.
Instead of relying on heuristic exploration, OSCAR guides VLMs with optimization-derived trajectories from the golden library, leading to more effective tool-call planning.
\end{itemize}

%

\subsection{In-depth Analysis}

\subsubsection{Ablation Study}
Table~\ref{tab:ablation_all_dataset} analyzes the impact of golden trajectory demonstrations and set-theoretic operations in OSCAR, from which we can have the following observations:
\begin{itemize}[leftmargin=10pt]
    \item \textbf{Trajectory demonstrations}.
    Removing the optimal trajectory demonstrations at inference time (\textbf{w/o Demo.}) generally degrades the performance, indicating that the optimization-derived trajectories provide valuable signals about how retrieval tools should be selected and composed.
    Moreover, the golden library implicitly captures dataset- and query-specific patterns, including query intent styles, preferred composition strategies, and top-$k$ selection biases, enabling the planner to adapt its trajectory planning across distributions.
    
    \item \textbf{Set-theoretic operations}.
    Performance degrades substantially when removing the set-difference or demo guidance. In particular, removing both makes the process close to a simple union of retrieved results, which is akin to multi-channel retrieval. As shown in Table~\ref{tab:ablation_all_dataset}, this variant generally performs poorly, with particularly large drops at Recall@50 across datasets. Naively aggregating outputs from multiple atomic retrievals is insufficient, as union-only composition cannot filter out items that match negative attributes or properly control the candidate scope. 
    
\end{itemize}

\begin{table}[t]
\centering

\resizebox{0.48\textwidth}{!}{
\begin{tabular}{c|cccc}
\toprule
\hline
Model & m@25 & m@50 & R@25 & R@50 \\ \hline

QQMM-embed-v2 & 50.39 & 51.45 & 90.25 & 95.26 \\ 
Qwen3-Embed-8B & 18.98 & 19.88 & 53.12 & 65.62 \\ 
FiRE & 34.40 & 35.50 & - & - \\ 
$X^{R}$ & 35.46 & 36.50 & - & - \\ 
\hline

Qwen3-VL-4B & 59.22 & 59.33 & 88.62 & 88.75 \\
 w/ OSCAR & \textbf{61.18} & \textbf{61.91} & \textbf{93.75} & \textbf{96.12} \\
 \rowcolor{purple!15} Rel.Improv. (\%) & 3.31\% & 4.35\% & 5.79\% & 8.30\% \\ 
\hline

Qwen3-VL-8B & 59.22 & 59.31 & 88.12 & 88.25 \\
 w/ OSCAR & \textbf{61.65} & \textbf{62.40} & \textbf{94.88} & \textbf{97.12} \\
 \rowcolor{purple!15} Rel.Improv. (\%) & 4.10\% & 5.21\% & 7.68\% & 10.05\% \\ 
\hline

Qwen3-VL-32B & 59.72 & 59.84 & 86.88 & 87.00 \\ 
 w/ OSCAR & \textbf{61.92} & \textbf{62.67} & \textbf{94.62} & \textbf{97.50} \\
\rowcolor{purple!15} Rel.Improv. (\%) & 3.68\% & 4.73\% & 8.91\% & 12.07\% \\ 
\hline

InternVL3.5-38B & 60.29 & 60.62 & 90.88 & 91.75 \\
 w/ OSCAR & \textbf{61.20} & \textbf{61.58} & \textbf{92.38} & \textbf{93.75} \\ 
\rowcolor{purple!15} Rel.Improv. (\%) & 1.51\% & 1.58\% & 1.65\% & 2.18\% \\ 
\hline
\bottomrule
\end{tabular}}
\caption{Performance of different VLM backbones on \textit{CIRCO} dataset.}
\label{tab:backbone}
\vspace{-2mm}
\end{table}

\subsubsection{Generalization of OSCAR}
\label{appendix:generalization}

We further investigate the generalization of our proposed OSCAR in terms of different VLM planner backbones. 
We evaluate with Qwen3-4B/8B/32B and InternVL3.5-38B on the CIRCO dataset, and report the results in Table~\ref{tab:backbone}. We obtain the following observations:
\begin{itemize}[leftmargin=10pt]
    \item OSCAR generalizes well across different VLM backbones and consistently delivers substantial performance gains. This suggests that our offline-derived golden library captures broadly applicable planning logic, enabling OSCAR to function as a training-free, plug-and-play framework with strong model compatibility.
    \item OSCAR’s gains over the VLM backbone largely rely on the VLM’s agentic reasoning and tool-calling abilities. Accordingly, a stronger VLM (\eg, Qwen3-VL-32B) may better exploit the golden library and yield larger improvements. This trend is also evident in our preliminary experiments with Qwen2.5-VL~\cite{qwen25vl} and MiniCPM-V~\cite{minicpmv}: models with weaker tool-calling capabilities often fail to follow tool-use instructions and cannot reliably execute the required calls, leading to degraded performance.
\end{itemize}

\subsection{Additional Analysis}
\label{sec:additional_analysis}
Due to the page limitation, we move all other experiments to Appendix G,
including the case study, robustness to the number of demonstrations,
industrial-gallery evaluation, CIRCO Recall evaluation, efficiency and
scalability analysis, OOD generalization analysis, tool combination analysis,
and the controlled experiment.

\section{Conclusion}

We present OSCAR, an optimization-steered planning framework that reframes agentic composed image retrieval from a heuristic search process into a principled trajectory planning problem.
By solving a two-stage MIP formulation offline, OSCAR mathematically derives optimal planning strategies, including query rewriting, truncation selection, and boolean set composition, without human annotation.
These optimal trajectories are stored as golden demonstrations to effectively steer the VLM planner during inference to replicate complex reasoning logic in a single pass.
Extensive experiments on three public benchmarks and private industrial user galleries demonstrate that OSCAR consistently outperforms SOTA baselines, offering both superior accuracy and robust generalization.
Future work will explore extending this optimization-steered paradigm to other complex reasoning tasks, such as multi-hop question answering and open-ended tool use.


\section*{Limitations}
OSCAR is designed for retrieval tasks, especially those that can be naturally decomposed into atomic retrievals and set-theoretic composition. 
For non-retrieval tasks, applying OSCAR would require redefining the action space, optimization objective, and composition logic, which is beyond the scope of this work.

\section*{Acknowledgments}
This paper is supported by National Natural Science Foundation of China (624B2096, 72595872, 72542012, 62322603).

\bibliography{custom}

\appendix
\newpage
\appendix

\clearpage
\section{Extended Related Works}
\label{sec:related_work_full}

\subsection{Unified Embedding Retrieval for CIR}

TIRG~\cite{fusion1} introduces a gated residual composition mechanism, where the reference image embedding is modified via text-conditioned residual and gating functions.
Building on VLMs, ~\citet{fusion2} propose a conditioned composition framework that combines CLIP-based visual and textual features through a learned fusion module.
CLIP4CIR~\cite{clip4cir} builds upon CLIP by fine-tuning the vision--language encoders and learning a dedicated combiner network to fuse the reference image and modification text into a unified retrieval representation.
ARTEMIS~\cite{artemis} focuses on relative image descriptions and contrastive supervision to improve sensitivity to semantic changes, while still collapsing the composed query into a unified embedding.
DC-Net~\cite{dc-net} adopts a dual-branch transformation strategy, separately processing visual and textual features before fusing them into a single retrieval representation.

Pic2Word~\cite{pic2word} projects image features into the textual embedding space as pseudo-word tokens, enabling image--text composition through pretrained text encoders.
SEARLE~\cite{searle} further extends this idea by learning task-specific textual inversion or composition modules that encode visual modifications implicitly within a single embedding.
Context-I2W~\cite{context-i2w} further maps images to context-dependent pseudo-word tokens conditioned on the modification intent, improving zero-shot composed retrieval under diverse edits.
FTI4CIR~\cite{fit4cir} introduces fine-grained textual inversion by mapping an image into a subject-oriented token together with multiple attribute-oriented tokens, capturing finer visual details.

Beyond explicit composition mechanisms, some works focus on improving embedding learning strategies for CIR without altering the underlying retrieval formulation.
~\citet{knowledge_enhanced} enrich pseudo-word token mapping with an external database and an additional concept-alignment stream, improving fine-grained semantic alignment in the text embedding space.
In contrast to token-based mapping, SEIZE~\cite{semantic} adopts a training-free textual reasoning pipeline that generates diverse captions and leverages LLM-based compositional reasoning before retrieval.

Some works improve embedding learning for CIR without altering the composition mechanism.
LinCIR~\cite{lincir} adopts a language-only self-supervised training strategy for zero-shot CIR, while FiRE~\cite{fire} fine-tunes VLMs to enhance multimodal retrieval representations.

Within the unified embedding paradigm,
CIReVL~\cite{cirevl} proposes a training-free framework following this paradigm, while LDRE~\cite{ldre} explores LLM-based divergent reasoning and ensemble-style aggregation to improve coverage of possible semantic interpretations.
Caption-based approaches make intermediate reasoning steps explicit in the language space; however, their effectiveness depends heavily on the quality of caption generation and rewriting, and may suffer from information loss when compressing visual content into text.

Recent advances in large-scale multimodal embedding models have substantially improved zero-shot retrieval performance. However, evaluations such as MMEB~\cite{vlm2vec,vlm2vec_v2} show that no single embedding consistently performs well across diverse CIR scenarios; instead, models such as QQMM-embed~\cite{QQMM}, Rzenembed~\cite{rzenemb}, Ops-MM-embedding~\cite{ops} and VLM2Vec~\cite{vlm2vec} display complementary strengths, motivating the use of multiple embedding tools rather than a single retriever.

Despite architectural and training differences, all these approaches ultimately reduce CIR to the unified-embedding model search, where complex compositional constraints such as inclusion and exclusion are implicitly encoded in a continuous representation.

\subsection{Heuristic Agentic Retrieval for CIR}

Recent work has explored agentic formulations for CIR, where the retrieval process is decomposed into multiple steps with specialized roles (\eg, planning, retrieval, and correction), often involving iterative refinement or feedback.
AutoCIR~\cite{autocir} exemplifies this approach by structuring CIR as a multi-component pipeline that iteratively revises queries and retrieval results.
MRA-CIR~\cite{MRA-CIR} further advances this paradigm by introducing a multimodal reasoning agent that explicitly decouples intent reasoning from retrieval execution, bridging LLM-based logic and retrieval tools.
~\citet{Reason-before-retrieve} propose a reflective reasoning-before-retrieval paradigm for training-free zero-shot CIR, while $X^R$~\cite{xr} introduces a cross-modal agent framework that coordinates specialized tools through synergistic planning across visual and textual modalities.
While this line of work demonstrates the potential of multi-step reasoning, including generating alternative hypotheses and correcting failure cases, such agentic pipelines are typically iterative and computationally expensive due to repeated model interactions. 
Furthermore, these methods largely rely on heuristic exploration rather than explicit planning with optimality guarantees, leaving open the challenge of how to formalize and optimize tool execution under explicit global constraints.

\subsection{Optimization-Based Planning}
From a different perspective, Operations Research formulations model decision-making problems by explicitly defining decision variables, constraints, and objective functions, and solving them via combinatorial optimization such as MIP~\cite{mip1,mip2,bpp-search,mlprompt,llm4mip}.
This paradigm enables principled and interpretable decision-making with explicit optimization objectives, rather than relying on heuristic or trial-and-error search.

Our work adopts this optimization perspective by formulating tool-call trajectory generation for CIR as a two-stage MIP problem.
Each atomic retrieval corresponds to a discrete decision variable, while inclusion, exclusion, and composition constraints are explicitly modeled through set-level constraints.
This formulation allows retrieval trajectories to be generated with explicit optimality objectives, distinguishing our approach from heuristic agentic planning.

\section{Additional Method Discussions}
\label{appendix:method_discussion}

We clarify key design choices and implementation details to ensure the reproducibility and scientific rigor of the OSCAR framework.

\vspace{3pt}
\noindent\textbf{Rigorous Enforcement of Logical Dependencies}.
In Sections~\ref{sec:method:stage1} and \ref{sec:method:stage2}, we described the relationships between primary decision variables (\ie, $x_r$) and state variables (\eg, $c_i, t_f$, $z_i$) using logical implications to prioritize conceptual clarity.
In our actual implementation, these logical dependencies are strictly enforced via standard linear constraints (\eg, Big-M formulations and logical inequalities).
The complete, mathematically rigorous formulations, including all auxiliary constraints and bounds, are provided in Appendix~\ref{appendix:mip_details}.

\vspace{3pt}
\noindent\textbf{Rationale for Two-Stage Optimization}.
Theoretically, the recall and precision objectives could be unified into a monolithic MIP that solves for the global optimum in a single stage.
However, in practice, a single-stage formulation suffers from a severe \textit{combinatorial explosion}, rendering it computationally intractable for large-scale datasets.
Our two-stage design adopts a \textit{prune-and-refine} strategy: the first MIP acts as a high-recall filter to reduce the entire search space to a manageable candidate set $\mathcal{U}$, making the second composition MIP computationally tractable.

\vspace{3pt}
\noindent\textbf{Negative Results and Failed Explorations}.
To provide a comprehensive view of the problem landscape and benefit the research community, we include a detailed discussion of negative results in the Appendix~\ref{appendix:stage2_f1}, such as optimal-but-inefficient formulation for composition MIP and alternative trajectory retrieval strategies.
By documenting these explorations, we aim to highlight not only the effective design of OSCAR but also the non-trivial challenges inherent in optimizing agentic retrieval trajectories.

\begin{figure*}[ht]
  \centering
  \includegraphics[width=0.99\linewidth]{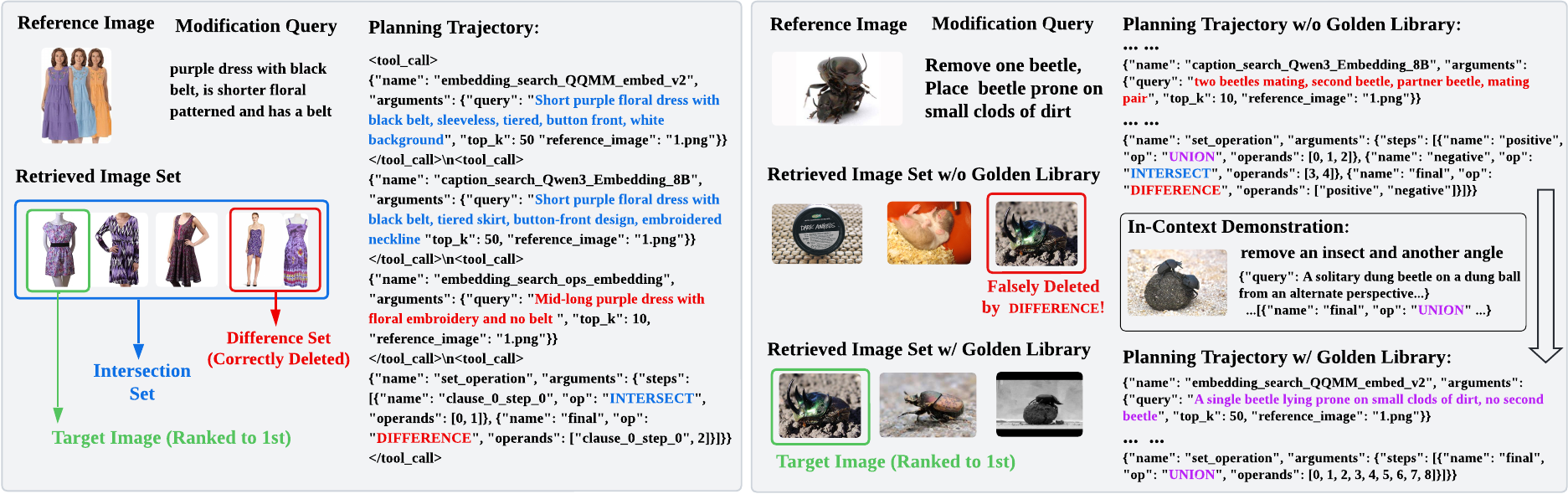}
  \vspace{-2mm}
  \caption{Case Studies on \textit{FashionIQ} (left) and \textit{CIRR} (right) datasets. The left part shows the correct tool call trajectory with the ground truth image ranked to the first place. The right part illustrates the effectiveness of golden library, with whose help the agent can avoid previous wrong tool calls and finally retrieve the ground truth image.}
  \label{fig:case_study}
\end{figure*}

\section{Dataset Statistics and Splits}
\label{appendix:dataset_stats}
Table~\ref{tab:dataset_info} summarizes the statistics of the datasets used in our experiments. Specifically, CIRR is split into a training set, a validation set, and a testing set, while CIRCO has no training set. FashionIQ is composed of three subsets (dress, shirt and toptee), each containing a training set and a validation set. Following previous works~\cite{autocir}, we evaluate on the FashionIQ validation set, and on the CIRR and CIRCO testing sets through the official remote evaluation server. We leverage the validation set of CIRCO, the training set of CIRR and FashionIQ to build the golden library (only 10\% of them). 

\begin{table}[t]
\centering

\resizebox{0.45\textwidth}{!}{
\renewcommand\arraystretch{0.96}
\begin{tabular}{cc|ccccc}
\toprule
\hline
\multirow{2}{*}{Split} & \multirow{2}{*}{Num} & \multirow{2}{*}{CIRCO} & \multirow{2}{*}{CIRR} & \multicolumn{3}{c}{FashionIQ} \\
                       &       &       &        & Dress & Shirt & Toptee \\ \hline
\multirow{2}{*}{Training} & \#Query & -     & 28,225 & 5,985 & 5,988 & 6,027  \\
                       & \#Image & -     & 16,939 & 11,452 & 19,036 & 16,121  \\ \hline
\multirow{2}{*}{Validation}   & \#Query & 220   & 4,181  & 2,017 & 2,038 & 1,961  \\
                       & \#Image & 123,403 & 2,297  & 3,817 & 6,346 & 5,373  \\ \hline
\multirow{2}{*}{Testing}  & \#Query & 800   & 4,148  & -     & -     & -      \\
                       & \#Image & 123,403   & 2,315 & -     & -     & -      \\ \hline
   \bottomrule
\end{tabular}}
\vspace{-5pt}

\caption{The dataset statistics.}
\label{tab:dataset_info}
\vspace{-10pt}
\end{table}

\section{Implementation Details}
\label{appendix:exp implementation}
For the tool planning process, we mainly use Qwen3-VL-32B~\cite{Qwen3-VL} as the planner agent. We adopt COPT~\cite{copt} as the MIP solver and
 discretize the top-$k$ parameter described in  Section~\ref{sec:method:candidate_query_generation} into the range from 5 to 50 with a step size of $5$. The detailed prompt for the planner is demonstrated in Appendix~\ref{appendix:prompts}. 
 We adopt Qwen3-VL-32B as the verifier agent to perform a binary relevance judgment between the candidates and queries. A detailed explanation of this binary scoring formulation is provided in Appendix~\ref{appendix:yesno_scoring}.

 For the tools used by the planner, we select Ops-MM-v1-7B~\cite{ops}, RzenEmbed-v2-7B~\cite{rzenemb}, QQMM-embed-v2~\cite{QQMM} as embedding-based retrieval tools, and bge-m3~\cite{bge-m3}, Qwen3-Embed-4B~\cite{qwen3_embedding},  Qwen3-Embed-8B~\cite{qwen3_embedding} as caption-based retrieval tools. In addition, we introduce a dedicated tool that performs structured set operations on retrieved result sets (details are provided in Appendix~\ref{appendix:set_operation}).
All retrieval and composition tools are deployed via MCP servers~\cite{mcp}. All experiments are conducted on A100 GPUs. 

To emphasize generalization, we construct the golden library using only a limited
subset of the available training data: specifically, we generate 220 cases for CIRCO,
and sample 10\% of the training set to construct cases for CIRR and FashionIQ.

\section{Analysis of Negative Filtering}
\label{appendix:negative_filtering}

\subsection{Statistics of Candidate Pruning}
\label{appendix:candidates_pruning}
Table~\ref{tab:diff_stats} reports the average number of candidates removed by the difference operation across different datasets.
When applied appropriately, the difference operation can effectively reduce the size of the candidate set, thereby alleviating the computational burden on the verifier.

\begin{table}[h]
\centering

\renewcommand\arraystretch{1.1}
\begin{tabular}{lccc}
\toprule
 Dataset  & Avg. Removed \\
\midrule
 CIRR  & 10.20 \\
 CIRCO  & 3.69 \\
 FashionIQ-Shirt & 3.67 \\
 FashionIQ-Dress& 3.05 \\
FashionIQ-Toptee & 2.51 \\
\bottomrule
\end{tabular}
\caption{Statistics of candidate pruning by the difference operation.
Avg. Removed denotes the average number of candidates excluded when difference is used.
}
\vspace{-5pt}

\label{tab:diff_stats}
\end{table}

\begin{table}[t]
\centering

\resizebox{0.48\textwidth}{!}{
\renewcommand{\arraystretch}{1.1}
\begin{tabular}{c|cc|cc|cc}
\hline
\multirow{2}{*}{Negative Filtering Strategy}
& \multicolumn{2}{c|}{CIRCO}
& \multicolumn{2}{c|}{CIRR}
& \multicolumn{2}{c}{FIQ.Avg} \\
& m@25 & m@50
& R@10 & R@50
& R@10 & R@50 \\
\hline
Negative Intersection (Ours)
& \textbf{61.92} & \textbf{62.67}
& \textbf{87.45} & \textbf{96.56}
& \textbf{43.73} & \textbf{67.97} \\

Negative Majority
& 59.13 & 59.54
& 85.92 & 92.81
& 42.96 & 66.54 \\

Negative Union
& 57.26 & 57.91
& 80.14 & 89.36
& 40.82 & 61.29 \\
\hline
\end{tabular}
}
\vspace{-5pt}

\caption{
Comparison of different negative filtering strategies.
Negative intersection consistently outperforms the more aggressive majority
and union strategies across all benchmarks.
}
\label{tab:negative_filtering}
\end{table}
\subsection{Comparison of Negative Filtering Strategies}

We further compare negative intersection with two more aggressive strategies:
negative majority and negative union. Negative majority removes an image if it
appears in more than half of the selected negative retrievals, while negative
union removes it if it appears in any selected negative retrieval.

As shown in Table~\ref{tab:negative_filtering}, negative intersection
consistently performs best, while more aggressive filtering leads to larger
performance degradation. Negative retrievals can be noisy and may mistakenly
retrieve ground-truth candidates; aggressive filtering therefore increases the
risk of removing relevant candidates before verification, after which they
cannot be recovered by the final ranker.

\section{Set Operations}
\label{appendix:set_operation}

We consider three types of set operations: union, intersection, and difference.
These operations enable explicit composition of results returned by multiple retrieval tools, producing a concise candidate set for subsequent ranking.
Specifically, union is used to improve recall by aggregating candidates from different tools, intersection increases precision by retaining only commonly retrieved results, and difference filters out candidates that are likely to be irrelevant.

\begin{figure*}[ht]
  \centering
  \includegraphics[width=0.95\linewidth]{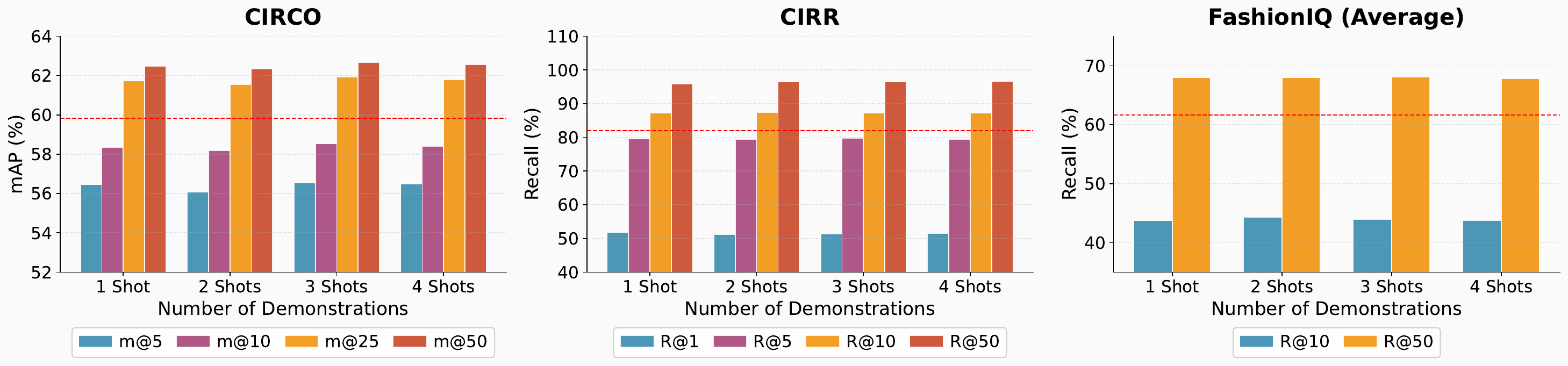}
  \vspace{-3mm}
  \caption{Performance comparison w.r.t. different numbers of demonstrations (\ie, the number of shots) for inference-time steering. ``m" and ``R" denotes mAP and Recall. The red dashed line denotes the zero-shot performance of OSCAR (\ie, mAP@50 on \textit{CIRCO}, Recall@50 on \textit{CIRR} and \textit{FashionIQ}). 
  }
  \label{fig:shot comparison}
\end{figure*}

\section{More Experiments}
\label{appendix:addition_experiment}

\subsection{Case Study}
\label{appendix:case_study}
We present the case study on FashionIQ and CIRR datasets in Figure~\ref{fig:case_study}. 
The left panel shows a successful planning trajectory where structured set operations retrieve the target image and rank it in the first place. 
The right panel highlights the effect of the golden library: based on the in-context golden demonstration, the planner is steered towards more appropriate tool-call and set-operation choices and consequently retrieves the ground-truth image.

A common failure mode we observe is a \textit{polarity mismatch in negative evidence}. The negative branch is intended to retrieve the undesired concept (\eg, ``more than one beetle'' in Figure~\ref{fig:intro}) and remove it via \textsc{Set Difference}. However, the planner may instead retrieve the negated concept (\eg, ``one beetle'') and mistakenly use it as the set to subtract. This reverses the intended set semantics and can remove valid targets.
With our optimization-derived guidance, the planner is steered towards consistent negative evidence and set composition, preventing such erroneous exclusion.

\begin{table}[t]
\centering

\resizebox{0.49\textwidth}{!}{
\renewcommand\arraystretch{1.15}
\begin{tabular}{ccccccccc}
\toprule
\hline
                        & \multicolumn{2}{c}{Gallery1} & \multicolumn{2}{c}{Gallery2} & \multicolumn{2}{c}{Gallery3} & \multicolumn{2}{c}{Avg}   \\ \cline{2-9} 
\multirow{-2}{*}{Model} & N@10         & R@10        & N@10         & R@10        & N@10         & R@10        & N@10        & R@10        \\ \hline

VLM2Vec                 & 44.29        & 47.78       & 44.91        & 48.69       & 38.96        & 41.86       & 42.72       & 46.11       \\
B3\_Qwen2\_7B           & 42.87        & 47.32       & 41.84        & 45.71       & 37.49        & 40.76       & 40.73       & 44.60       \\
Ops-MM-v1               & 51.78        & 53.74       & 48.95        & 51.44       & 43.98        & 46.99       & 48.24       & 50.72       \\
RzenEmbed-v2-7B         & 51.04        & 53.79       & 49.62        & {\ul 51.93} & 46.23        & 47.59       & 48.96       & 51.10       \\
QQMM-embed-v2           & 51.64        & 54.43       & 48.42        & 51.29       & 44.60        & 46.63       & 48.22       & 50.78       \\ \hline

bge-m3                  & 39.61        & 41.98       & 38.36        & 41.47       & 39.87        & 41.41       & 39.28       & 41.62       \\
Qwen3-Embed-0.6B        & 40.73        & 42.78       & 40.11        & 42.54       & 39.60        & 41.11       & 40.15       & 42.14       \\
Qwen3-Embed-4B          & 43.06        & 44.93       & 44.55        & 46.80       & 41.48        & 43.07       & 43.03       & 44.93       \\
Qwen3-Embed-8B          & 42.92        & 45.05       & 44.57        & 46.86       & 41.89        & 43.13       & 43.13       & 45.01       \\ \hline

$X^R$                   & {\ul 52.49}  & {\ul 56.17} & {\ul 50.23}  & 48.72       & {\ul 49.74}  & {\ul 56.91} & {\ul 50.82} & {\ul 53.93} \\ \hline

OSCAR (Ours)            & \textbf{56.01} & \textbf{65.28} & \textbf{53.12} & \textbf{57.43} & \textbf{58.70} & \textbf{63.92} & \textbf{55.94} & \textbf{62.21} \\
\rowcolor{purple!15} Rel.Improve. (\%) 
                         & 6.71\%       & 16.22\%      & 5.75\%       & 10.59\%      & 18.01\%      & 12.32\%      & 10.08\%      & 15.36\%      \\ \hline
\bottomrule
\end{tabular}}
\caption{Performance of real-world text-to-image retrieval on private industrial photo galleries. N@$K$ and R@$K$ denote NDCG@$K$ and Recall@$K$, respectively. Best results are highlighted in bold, while the second best are underlined.}
\label{tab:album}

\end{table}

\subsection{Evaluation on Industrial Photo Galleries}
\label{appendix:industrial_gallery}
To further assess the robustness of our OSCAR under real-world industrial scenarios, we evaluate it on a private in-the-wild dataset consisting of multiple user photo galleries. Each gallery corresponds to an individual user’s photo collection and is associated with text queries for within-gallery image retrieval. 
This setting reflects practical album search scenarios, and the task becomes general text-to-image retrieval, where image distributions are highly diverse and queries often require non-trivial reasoning over user intentions beyond simple visual descriptions. 
Specifically, we test on three galleries composed of 1,069/1,466/1,047 real-world images and 483/336/369 real-user queries.

We report the performance in Table~\ref{tab:album}.
We can observe that OSCAR consistently outperforms strong  baselines across all galleries in terms of both metrics.
On average, OSCAR achieves a relative improvement of 6.71\% in NDCG@10 and 16.22\% in Recall@10 over the best baseline.
These results demonstrate that our optimization-steered tool-call trajectory planning and structured set-theoretic composition remain effective beyond public benchmarks, and generalize well to real-world user photo collections with natural data distributions and complex user intentions.

\subsection{Robustness to the Number of Demonstrations}

We study how the number of retrieved demonstrations from the golden library (\ie, the number of shots) affects inference-time steering of the VLM planner. Specifically, we vary the number of demonstrations in \{1, 2, 3, 4\} and evaluate CIR performance under otherwise identical settings. As shown in Figure~\ref{fig:shot comparison}, OSCAR is largely insensitive to the number of demonstrations. The performance remains stable across all datasets as the shot count increases, with only minor fluctuations. This again suggests that the golden library primarily provides \textit{generalizable planning trajectories}, rather than merely memorizing answers to similar queries. Consequently, adding more demonstrations yields limited additional information beyond a single trajectory.

\subsection{CIRCO Recall}
\label{appendix:circo_recall}

In addition to the widely used mAP metric on CIRCO, we also report Recall@K in Table~\ref{tab:circo_recall}. As shown in the table, OSCAR consistently
achieves higher recall than unified embedding baselines, indicating its ability to retrieve a larger portion of target images.

\begin{table}[t]
\centering

\resizebox{0.48\textwidth}{!}{
\renewcommand\arraystretch{0.9}
\begin{tabular}{c|cccc}
\toprule
\hline
Method               & R@5            & R@10           & R@25           & R@50           \\ \hline
Ops-MM-v1-7B         & 37.88          & 54.25          & 69.75          & 78.00          \\
RzenEmbed-v2-7B      & 60.75          & 72.38          & 85.75          & 92.25          \\
VLM2Vec              & 11.00          & 16.38          & 26.75          & 38.62          \\
B3\_Qwen2\_7B        & 12.00          & 18.00          & 29.00          & 38.75          \\
QQMM-embed-v2        & {\ul 71.12}    & {\ul 79.75}    & {\ul 90.25}    & {\ul 95.26} \\ \hline
bge-m3               & 14.62          & 20.62          & 32.25          & 40.62          \\
Qwen3-Embed-0.6B     & 17.12          & 25.00          & 38.00          & 50.62          \\
Qwen3-Embed-4B       & 25.62          & 37.38          & 51.25          & 62.62          \\
Qwen3-Embed-8B       & 27.25          & 38.75          & 53.12          & 65.62          \\
\hline
\textbf{OSCAR (Ours)} & \textbf{75.50} & \textbf{85.50} &  \textbf{94.62} & \textbf{97.50}   \\ \hline
\rowcolor{purple!15}
Rel.Impr (\%)        & 6.16           & 6.72           & 4.84           & 2.35          \\ \hline
   \bottomrule
\end{tabular}}
\vspace{-5pt}
\caption{Recall comparison on \textit{CIRCO} dataset. Best results are highlighted in bold, while the second best are underlined. \textit{Rel.Impr} denotes the relative improvement over the best baseline.}
\label{tab:circo_recall}
\end{table}

\subsection{Efficiency Analysis}

\begin{table}[t]
\centering

\resizebox{0.49\textwidth}{!}{

\begin{tabular}{c|cccc}
\toprule
\hline
Dataset & \makecell{Atomic Retrieval \\ Construction} & \makecell{Recall-Oriented \\ Selection MIP} & \makecell{Precision-Oriented \\ Composition MIP} & Total \\ \hline
CIRCO & 12.03s & 0.58s & 2.62s & 15.23s \\
CIRR & 13.45s & 0.21s & 0.28s & 13.94s \\
FashionIQ & 19.44s & 0.45s & 2.00s & 21.89s \\ \hline
\bottomrule
\end{tabular}

}
\vspace{-5pt}

\caption{Offline time cost analysis.}
\label{tab:offline efficiency}
\end{table}

\begin{table}[t]
\centering

\resizebox{0.49\textwidth}{!}{
\begin{tabular}{c|ccccc}
\toprule
\hline
Dataset & \makecell{Case \\ Retrieval} & Plan & \makecell{Candidate \\ Retrieval} & Verification & Total \\ \hline
CIRCO & 0.96s & 7.93s & 2.10s & 12.79s & 23.78s \\
CIRR & 1.05s & 11.53s & 0.72s & 10.84s & 24.14s \\
FashionIQ & 1.08s & 12.99s & 1.37s & 5.03s & 20.47s \\ \hline
\bottomrule
\end{tabular}}
\vspace{-5pt}

\caption{Online testing time cost analysis.}
\label{tab:online efficiency}

\end{table}
\begin{table}[t]
\centering

\resizebox{0.45\textwidth}{!}{
\begin{tabular}{c|cc}
\hline
Method & CIRCO (s) & CIRR (s)\\
\hline
AutoCIR & 33.9 & 29.4\\
XR & 36.5 & 43.0\\
\textbf{OSCAR (Ours)} & \textbf{23.8} & \textbf{24.1}\\
\hline
\end{tabular}
}
\vspace{-5pt}

\caption{
Inference efficiency comparison with agentic baselines.
Under the same planner, verifier, and retrieval settings, OSCAR achieves the lowest latency among all methods, outperforming AutoCIR and XR in both CIRCO and CIRR benchmarks.
}
\label{appendix:agentic_efficiency}
\end{table}

We provide a breakdown of the time cost for both the offline and online stages of OSCAR, averaged per sample. The experiments are conducted on 4 A100 GPUs. The planner and verifier are Qwen3-VL-32B-Instruct deployed with vLLM, and the atomic retrieval tools (bge, Qwen3-Emb-4B/8B, QQMM, Ops, and RzenEmbed) are executed in parallel.

\noindent \textbf{Offline Stage.} Table~\ref{tab:offline efficiency} shows that the dominant cost comes from atomic retrieval construction, while the two MIP stages are relatively lightweight. This supports our claim that the two-stage offline design avoids combinatorial explosion in practice.

\noindent \textbf{Online Stage.} Table~\ref{tab:online efficiency} shows that the main cost comes from the planning and verification stages, while case retrieval and candidate retrieval contribute only a small portion of the total latency. This indicates that OSCAR’s online overhead is dominated by VLM reasoning, rather than by the retrieval or library lookup itself.

\noindent B\textbf{asline Comparison.} We further compare OSCAR with iterative agentic baselines under a fair setting using the same planner, verifier, and QQMM retriever. As shown in Table ~\ref{appendix:agentic_efficiency}, on CIRCO, OSCAR reduces the per-sample latency from 33.9s (AutoCIR) and 36.5s (XR) to 23.8s; on CIRR, it reduces the latency from 29.4s  (AutoCIR) and 43.0s (XR) to 24.1s. Generally, OSCAR remains more efficient than iterative agentic pipelines, as it replaces multi-round trial-and-error interaction with a optimization-steered compositional retrieval plan.

\paragraph{Scalability Analysis.}
We further evaluate the scalability of the offline optimization by enlarging
two key dimensions of the atomic retrieval space: the number of retrieval
tools and the granularity of top-$k$ discretization. Starting from the default
setting of 6 tools and 10 top-$k$ cutoffs, we increase the tool pool to 10
and the number of cutoffs to 20. As MIP is NP-hard, the solving cost naturally
increases as the atomic retrieval space grows. As shown in
Table~\ref{tab:mip_scalability}, our two-stage formulation mitigates this
growth through a prune-and-refine strategy, where Stage I first reduces the
search space and Stage II performs composition only over the reduced candidate
space.

\begin{table*}[t]
\centering

\resizebox{\textwidth}{!}{
\renewcommand{\arraystretch}{1.1}
\begin{tabular}{llccccc}
\hline
Setting
& \# Retrieval Tools
& Top-$k$ Settings
& Stage I MIP
& Stage II MIP
& Total MIP Time
& Increase vs. Default \\
\hline
Default
& 6 tools: 3 visual + 3 caption/text
& 10
& 0.58s
& 2.62s
& 3.20s
& -- \\

More tools
& 8 tools: 4 visual + 4 caption/text
& 10
& 0.82s
& 3.69s
& 4.51s
& +40.9\% \\

Max tools
& 10 tools: 5 visual + 5 caption/text
& 10
& 1.07s
& 4.85s
& 5.92s
& +85.0\% \\

Finer top-$k$
& 6 tools: 3 visual + 3 caption/text
& 20
& 1.33s
& 6.02s
& 7.35s
& +129.7\% \\

Max tools + finer top-$k$
& 10 tools: 5 visual + 5 caption/text
& 20
& 2.46s
& 11.10s
& 13.56s
& +323.8\% \\
\hline
\end{tabular}
}
\vspace{-5pt}

\caption{
Scalability analysis of the offline MIP optimization with different numbers of
retrieval tools and top-$k$ settings.
}
\label{tab:mip_scalability}
\end{table*}

\subsection{OOD Generalization}


\begin{table}[t]
\centering

\vspace{-3pt}
\resizebox{0.46\textwidth}{!}{
\renewcommand{\arraystretch}{1.1}
\begin{tabular}{ccc}
\hline
\makecell{Golden Library Set} & Testing Set & R@10 ($\Delta$) \\ 
\hline
CIRR      & CIRCO     & 85.25 (+5.50) \\
CIRCO     & CIRCO     & \textbf{85.50} (+5.75) \\
CIRCO     & FIQ-Dress & 38.17 (+0.79) \\
FIQ-Dress & FIQ-Dress & \textbf{38.47} (+1.09) \\
\hline
\end{tabular}
}
\vspace{-5pt}

\caption{
OOD generalization of the golden library.
We construct the golden library on one dataset and evaluate it on another.
Numbers in parentheses indicate the absolute improvement over the best baseline on the corresponding testing set.
}
\label{tab:OOD}
\end{table}

We evaluate cross-dataset generalization of the Golden Library. Specifically, we build the library on one dataset and test on another. We show the results in Table~\ref{tab:OOD}, which demonstrate that the performance gap between cross-dataset and in-domain settings is small, suggesting that the Golden Library captures \textbf{transferable planning patterns} rather than merely memorizing dataset-specific cues. It's noteworthy that the Golden Library only uses 10\% of the training data, further demonstrating the generalization ability.

\subsection{Tool Selection \& Pool Size Study}

\begin{table}[t]
\centering

\resizebox{0.49\textwidth}{!}{

\begin{tabular}{cc|cccc}
\toprule
\hline
Multimodal & Text & m@5 & m@10 & m@25 & m@50 \\ \hline
\makecell{QQMM, RzenEmb, \\ Ops-MM} & \makecell{Bge, Qwen3-4B,\\ Qwen3-8B} & \textbf{56.54} & 58.53 & \textbf{61.92} & \textbf{62.67} \\

\hline

\begin{tabular}[c]{@{}c@{}} QQMM, VLM2Vec,\\ B3\_Qwen2\_7B\end{tabular} & \makecell{Bge, Qwen3-4B, \\Qwen3-8B} & 55.98 & 57.87 & 59.85 & 60.54 \\

\hline

\makecell{QQMM,\\ RzenEmb, Ops-MM} & \makecell{Qwen3-0.6B,\\ Qwen3-4B, Qwen3-8B} & 56.47 & \textbf{58.62} & 61.63 & 62.59 \\ \hline

\bottomrule
\end{tabular}}
\vspace{-5pt}

\caption{Performance w.r.t. different tools on CIRCO.}
\label{tab:tool_selection}
\end{table}

\noindent \textbf{Tool Selection.} We provide an ablation study in Table~\ref{tab:tool_selection}. In all settings, we use a fixed pool of three multimodal retrieval tools and three text retrieval tools.
\begin{itemize}[leftmargin=10pt]
    \item \textbf{Row 1 (original setting).} We use our default tool pool: multimodal tools QQMM, RzenEmb, and Ops-MM, together with text tools BGE, Qwen3-Emb-4B, and Qwen3-Emb-8B.
    \item \textbf{Row 2 (multimodal replacement).} We replace the multimodal tool pool with a \textbf{weaker} set: QQMM, VLM2Vec, and B3-Qwen2-7B.
    \item \textbf{Row 3 (text-tool replacement).} We replace the text tool pool with a more \textbf{homogeneous Qwen3-only} set: Qwen3-Emb-0.6B, Qwen3-Emb-4B, and Qwen3-Emb-8B.    
\end{itemize}

The results suggest two main observations:
\begin{itemize}[leftmargin=10pt]
    \item \textbf{Tool capacity matters.} Stronger individual tools generally lead to better overall retrieval performance, as Row 1 consistently outperforms Row 2. 
    \item \textbf{Tool diversity also matters.} A heterogeneous tool pool is more effective for handling complex and diverse query intents, as Row 1 outperforms Row 3, even though BGE itself is weaker than Qwen3-Emb-0.6B.
\end{itemize}

\noindent \textbf{Tool Pool Size.} We further study the effect of tool pool size on CIRCO. Starting from a minimal pool of QQMM + Qwen3-8B, we progressively add one multimodal tool and one text tool at each step. We observe that m@25 improves from 58.82 → 59.73 → 61.92 → 61.98 as the pool size increases from 2 to 8. This shows that enlarging the tool pool is generally beneficial, as it provides more complementary capabilities for diverse queries. At the same time, the gain becomes marginal once the pool is sufficiently rich.

\subsection{Controlled Experiment}

\begin{table}[t]
\centering

\vspace{-3pt}

\resizebox{0.49\textwidth}{!}{
\renewcommand\arraystretch{1.1}

\begin{tabular}{c|cccc}
\toprule
\hline
Method & m@5 & m@10 & m@25 & m@50 \\ \hline
AutoCIR & 50.57 & 54.17 & 57.12 & 57.54 \\
$X^R$ & 52.39 & 55.29 & 57.75 & 58.83 \\
OSCAR (ours) & \textbf{56.54} & \textbf{58.53} & \textbf{61.92} & \textbf{62.67} \\ \hline
\bottomrule
\end{tabular}}
\vspace{-5pt}

\caption{Agentic methods w/ QQMM and Qwen3-VL-32B on CIRCO.}
\label{tab:controlled_exp}
\end{table}

We provide a controlled comparison where all agentic methods are equipped with the same VLM and retrieval tool. We note that, in CIR evaluation, it is common practice to report the original results of existing baselines from their papers~\cite{autocir, xr, circo}, since these methods are often designed and tuned with their own retrieval backbones and VLMs. Thus, for the main results in Table~\ref{tab:circo_cirr_main_table} and Table~\ref{tab:fashioniq_main_table}, we follow previous works and report the original number. However, to further rule out the possibility that OSCAR's gains come from using a stronger retrieval model or VLM, we conduct an additional controlled experiment. In our OSCAR framework, multiple retrieval models are treated as atomic tools. For a fair comparison, we provide the agentic baselines with the strongest single embedding model among OSCAR's atomic tools, and also control the VLM to be the same across all methods. As shown in Table~\ref{tab:controlled_exp}, OSCAR still consistently outperforms these agentic baselines under this controlled setting.

We would like to highlight that OSCAR's advantage is not driven by simply using a strong model such as QQMM. As illustrated in Figure~1, a single retrieval model suffers from \textbf{single-model myopia}: it may perform well for certain query types but fail to capture other complementary retrieval signals. OSCAR addresses this limitation by systematically \textbf{unifying} diverse retrieval tools through optimization-steered planning. Instead of relying on one strong embedding space, OSCAR decomposes the query into atomic retrieval intents, selects and composes the outputs of different retrieval tools. Therefore, the improvement comes from tool coordination and structured planning, rather than from merely replacing the backbone with a stronger model.

\subsection{Comparison with DNF-based Composition}
\label{appendix:rationale_dnf}

The rationale for our two-stage prune-and-refine optimization is discussed in
Appendix~\ref{appendix:method_discussion}.
We further investigate a more expressive Disjunctive Normal Form (DNF)-based
composition, whose detailed formulation is provided in
Appendix~\ref{appendix:stage2_f1}.
To validate this design choice, we directly compare our structured composition
with the DNF-based formulation in Table~\ref{tab:dnf_comparison}.

\begin{table}[t]
\centering

\resizebox{0.45\textwidth}{!}{
\begin{tabular}{cccc}
\hline
Dataset & Metric & DNF & OSCAR \\
\hline
CIRCO & m@5  & 49.47 & \textbf{56.54} \\
      & m@10 & 49.30 & \textbf{58.53} \\
      & m@25 & 50.68 & \textbf{61.92} \\
      & m@50 & 50.90 & \textbf{62.67} \\
\hline
CIRR  & R@1  & 41.18 & \textbf{51.18} \\
      & R@5  & 53.23 & \textbf{79.50} \\
      & R@10 & 63.76 & \textbf{87.45} \\
      & R@50 & 86.60 & \textbf{96.56} \\
\hline
FashionIQ & R@10 & 44.18 & \textbf{47.73} \\
          & R@50 & 53.43 & \textbf{67.97} \\
\hline
\end{tabular}
}
\vspace{-5pt}
\caption{
Comparison between DNF-based composition and OSCAR's structured composition.
}
\label{tab:dnf_comparison}
\end{table}
In practice, DNF is more sensitive to imperfect retrieval tools.
Each DNF clause relies on intersections of multiple retrieval results, where
negative retrievals are incorporated through their complements. Consequently,
a ground-truth candidate can be eliminated by either a missed positive
retrieval or an erroneous negative retrieval. Once the ground truth is not
preserved by any clause, it is removed before the final verification stage and
cannot be recovered by the ranker. In contrast, our structured composition
uses a union for positive evidence and a conservative intersection for negative
evidence, making candidate retention more robust to individual retrieval
errors.

Although DNF is theoretically more expressive, the results show that this
additional expressiveness does not translate into better retrieval performance
with noisy retrieval tools. For example, for the query
\textit{``find a family photo without my dad''}, a compact composition may be
written as
\[
\texttt{ToolA(family photo)} -
\texttt{ToolB(my dad)}.
\]
In contrast, DNF may introduce multiple conjunctive clauses, e.g.,
\[
\begin{aligned}
&\{\texttt{ToolA(family)}-\texttt{ToolB(dad)}\} \\
{}\cup{}&
\{\texttt{ToolA(me and my mom)}-\texttt{ToolC(dad)}\} \\
{}\cup{}&
\{\texttt{ToolD(grandparents)}-\texttt{ToolA(children)}\}.
\end{aligned}
\]
Within each clause, an error from any intersected positive or negative
retrieval can eliminate a relevant candidate. If the candidate is not retained
by another clause, it never reaches the final verifier for reranking.

\section{Term Clarification}

\subsection{On \textit{"training free"}}
We use  \textbf{"training free"} to denote having no parameter updates in this paper, as OSCAR utilizes training data solely to generate In-Context Learning (ICL) demonstrations. In the agentic literature, this can be viewed as a few-shot rather than zero-shot approach, but it remains completely parameter-frozen.

\subsection{On "\textbf{atomic tool}"}

Agentic CIR methods may involve diverse subtasks beyond directly applying a retrieval model to the query. In OSCAR, these functionalities are naturally integrated into the retrieval process, \eg, the planner may rewrite the query when generating tool calls, and each retrieval tool may internally include auxiliary procedures such as caption generation, depending on its design. OSCAR then abstracts the outputs of these heterogeneous tools as atomic retrieval sets for subsequent selection and composition.






\clearpage
\onecolumn

\section{Implementation Details of two-stage MIPs}
\label{appendix:mip_details}
The common notations are summarized in Table~\ref{tab:notation}.

\subsection{Recall-Oriented Atomic Retrieval Selection}

We introduce binary decision variables $x_r$ to indicate if a positive atomic retrieval
$r \in \mathcal{R}^+$ is selected.
For image coverage, we use binary indicators $y_i$ for $i \in \mathcal{I}^+$ and
$z_i$ for $i \in \mathcal{I}^-$ to denote whether a ground-truth image or a non-ground-truth image,
respectively, is retrieved by at least one selected atomic retrieval.

Moreover, we group atomic retrievals into different \textit{families} $\mathcal{F}$. A family $F \in \mathcal{F}$ consists of retrievals that share the same tool $f$, query $\hat{q}$, and polarity $p$, differing only in their truncation threshold $k$.
We also introduce binary variables $t_f$ to indicate whether any selected atomic
retrieval is associated with tool $f$.

Based on the incidence definitions in Table~\ref{tab:notation},
we define the following upper-bound coefficients:
\[
\hat{a}_i := \sum_{r \in \mathcal{R}^+} a_{ir},
\qquad
\hat{b}_i := \sum_{r \in \mathcal{R}^+} b_{ir},
\]
which respectively count the number of positive atomic retrievals
that retrieve image $i$ as a target or non-target instance.

The recall-oriented atomic retrieval selection is formulated in the following MIP: 
\begin{align} \max_{x,y,z,t}\quad & w_R \frac{1}{|\mathcal{I}^+|}\sum_{i\in\mathcal{I}^+} y_i - w_F \frac{1}{|\mathcal{I}^-|}\sum_{i\in\mathcal{I}^-} z_i + \lambda_{\mathrm{div}}\sum_f t_f \tag{1a}\label{eq:stage1_obj}\\ \text{s.t.}\quad & y_i \le \sum_{r\in\mathcal{R}^+} a_{ir} x_r \le \hat{a}_i\, y_i, \qquad \forall i\in\mathcal{I}^+, \tag{1b}\label{eq:stage1_gt}\\ & z_i \le \sum_{r\in\mathcal{R}^+} b_{ir} x_r \le \hat{b}_i\, z_i, \qquad \forall i\in\mathcal{I}^-, \tag{1c}\label{eq:stage1_ngt}\\ & \sum_{r\in F} x_r \le 1, \qquad \forall F\in\mathcal{F}, \tag{1d}\label{eq:stage1_family}\\ & t_f \le \sum_{r\in\mathcal{T}_f} x_r \le |\mathcal{T}_f|\, t_f, \qquad \forall f, \tag{1e}\label{eq:stage1_type}\\ & x_r,\, y_i,\, z_i,\, t_f \in \{0,1\}. \tag{1f}\label{eq:stage1_bin} \end{align}

The objective in~\eqref{eq:stage1_obj} prioritizes recall to retain sufficient
candidates for subsequent composition, with diversity serving as a tie-breaker.

Constraints~\eqref{eq:stage1_gt}--\eqref{eq:stage1_ngt} ensure that
an image is treated as retrieved if it is returned by at least one selected positive atomic retrieval.

Constraint~\eqref{eq:stage1_family} therefore selects at most one variant per family,
avoiding redundant inclusion of nested result sets.

Constraint~\eqref{eq:stage1_type} introduces a mild tool-type diversity bias,
encouraging solutions that draw from multiple retrieval tools.

The solution to this MIP yields the optimal set of positive atomic retrievals, denoted as $\mathcal{R}^{+}_{*}$, which produces the recall-oriented candidate set of images $\mathcal{U} = \bigcup_{r \in \mathcal{R}^{+}_{*}} \mathcal{S}_r$.

\begin{table}[t]
\centering

\setlength{\tabcolsep}{6pt}
\renewcommand{\arraystretch}{1.1}
\begin{tabular}{p{0.2\linewidth} p{0.70\linewidth}}
\toprule
\textbf{Symbol} & \textbf{Description} \\
\midrule

$\mathcal{I}$
& An image gallery \\

$\mathcal{R}^+$ 
& Set of positive atomic retrievals\\

$\mathcal{R}^-$ 
& Set of negative atomic retrievals. \\

$r \in \mathcal{R}^+ \cup \mathcal{R}^-$
& An atomic retrieval. \\

$\mathcal{S}_r$
& Set of candidate images returned by atomic retrieval call $r$, where $\mathcal{S}_r \subseteq \mathcal{I}$. \\

$\mathcal{I}^+$
& Set of ground-truth images, $\mathcal{I}^+ \subseteq \mathcal{I}$. \\

$\mathcal{I}^-$
& Set of non-ground-truth images,
$\mathcal{I}^- = \mathcal{I}\setminus \mathcal{I}^+$. \\

$i \in \mathcal{I}$
& A candidate image. \\

$a_{ir}$ 
& $a_{ir} = \mathbf{1}[\, i \in \mathcal{I}^+ \cap \mathcal{S}_r \,], \; r \in \mathcal{R}^+. $ \\

$b_{ir}$ 
& $b_{ir} = \mathbf{1}[\, i \in \mathcal{I}^- \cap \mathcal{S}_r \,], \; r \in \mathcal{R}^+. $ \\

$\mathcal{R}^{+}_{*}$ & Set of positive atomic retrievals from Stage I.\\
$\mathcal{U}$
& Positive candidate universe induced by the selected positive retrievals from Stage I. \\

$c_{ir}$
& $c_{ir} = \mathbf{1}[\, i \in \mathcal{U} \cap \mathcal{S}_r \,], \; r \in \mathcal{R}^+_*. $ \\

$d_{ir}$
& $d_{ir} = \mathbf{1}[\, i \in \mathcal{U} \cap \mathcal{S}_r \,], \; r \in \mathcal{R}^-. $ \\

$\mathcal{R}^{+}_{**}$ & Set of positive atomic retrievals from Stage II.\\
$\mathcal{R}^{-}_{*}$ & Set of negative atomic retrievals from Stage II.\\
\bottomrule
\end{tabular}
\caption{Notation for the Problem Formulation}
\label{tab:notation}
\end{table}

\subsection{Precision-Oriented Logic Composition}

We use binary variables $x_r$ and $w_r$ to indicate if a positive or negative
atomic retrieval $r \in \mathcal{R}^+_*$ or $r \in \mathcal{R}^-$ is selected, respectively.
For each image $i \in \mathcal{U}$, the binary variable $e_i$ indicates whether $i$
is included in the final composed result.
We further introduce an auxiliary binary variable $g_i$ to indicate whether image $i$
lies in the intersection of the selected negative atomic retrieval result sets.

The precision-oriented logic composition is formulated in the following MIP: \begin{align}
\max_{x,w,r,g}\quad &
\sum_{i \in \mathcal{U} \cap \mathcal{I}^+} e_i
\tag{2a}\label{eq:stage2_obj_1}
\\
\text{s.t.}\quad
& e_i \le \sum_{r \in \mathcal{R}_*^+} c_{ir} \, x_r,
\qquad \forall i \in \mathcal{U}
\tag{2b}\label{eq:stage2_pos}\\
& e_i \le 1 - g_i,
\qquad \forall i \in \mathcal{U}
\tag{2c}\label{eq:stage2_diff}\\
& g_i \le d_{ir} + (1 - w_r),
\qquad \forall i \in \mathcal{U},\ \forall r \in \mathcal{R}^-
\tag{2d}\label{eq:stage2_int_ub}\\
& g_i \ge 1 - \sum_{r \in \mathcal{R}^-} (1 - d_{ir})\, w_r,
\qquad \forall i \in \mathcal{U}
\tag{2e}\label{eq:stage2_int_lb}\\
& \sum_{r \in \mathcal{R}^+_*} x_r \ge 1
\tag{2f}\label{eq:stage2_nonempty}\\
& x_r,\, w_r,\, e_i,\, g_i \in \{0,1\}.
\tag{2g}\label{eq:stage2_bin}
\end{align}

The objective~\eqref{eq:stage2_obj_1} is precision-oriented by design and
encourages retaining as many target images as possible after filtering.

Constraint~\eqref{eq:stage2_pos} allows an item to be included in the final result
if it is retrieved by at least one selected positive atomic retrieval.

Constraint~\eqref{eq:stage2_diff} implements exclusion by preventing items that
satisfy the negative clause from being returned.

Constraints~\eqref{eq:stage2_int_ub}--\eqref{eq:stage2_int_lb} ensure that
$g_u = 1$ if and only if image $u$ is retrieved by all selected negative atomic retrievals.

Constraint~\eqref{eq:stage2_nonempty} prevents degenerate solutions by enforcing that
at least one positive tool is selected, ensuring a non-empty positive clause
in the two-clause composition.

The solution yields an optimal plan $(\mathcal{R}^{+}_{**}, \mathcal{R}^{-}_{*})$ consisting of the selected positive and negative atomic retrievals, which can be used to construct the golden library.

\section{Alternative Stage II: F1-Optimized DNF Composition}
\label{appendix:stage2_f1}

Stage II adopts a fixed set-composition structure for stability.
Here, we consider a more flexible alternative that represents logical composition in
disjunctive normal form (DNF) and selects clauses via an MIP to directly optimize an F1-style objective.

To handle negative atomic retrievals, we convert set difference into intersection via relative complements.
Specifically, for any two sets $A$ and $B$ with $A,B \subseteq \mathcal{U}$, we apply the identity
\begin{equation}
A \setminus B
\;=\;
A \cap (\mathcal{U} \setminus B)
\;=\;
A \cap \overline{B},
\label{eq:diff_to_intersect}
\end{equation}
where $\mathcal{U}$ denotes the positive candidate universe induced by Stage I,
and $\overline{B}$ denotes the relative complement of $B$ with respect to $\mathcal{U}$,
\ie, $\overline{B} := \mathcal{U}\setminus B$.

Accordingly, each clause $C_c$ is defined as
\begin{equation}
C_c
=
\bigcap_{r \in P_c} \tilde{\mathcal{S}}_r,
\qquad
\tilde{\mathcal{S}}_r =
\begin{cases}
\mathcal{S}_r, & r \in \mathcal{R}^+_*,\\
\mathcal{U} \setminus \mathcal{S}_r, & r \in \mathcal{R}^-,
\end{cases}
\label{eq:clause_def}
\end{equation}
where $P_c \subseteq \mathcal{R}^+_* \cup \mathcal{R}^-$ denotes the set of atomic retrievals included in clause $c$.

We enumerate clauses prior to optimization by imposing constraints on the maximum
clause length and the number of negative atomic retrievals per clause, yielding a
finite candidate set $\{C_c\}_{c=1}^C$.
We introduce a binary indicator $u_c$ to denote whether clause $c$ is selected, and
define the final retrieval result as the union of selected clauses:
\begin{equation}
\mathcal{O}
=
\bigcup_{c:u_c=1} C_c .
\label{eq:dnf_union}
\end{equation}

We define clause-level incidence matrices
\[
\alpha_{ic} = \mathbbm{1}[\, i \in \mathcal{I}^+ \cap C_c \,],
\qquad
\beta_{ic} = \mathbbm{1}[\, i \in \mathcal{I}^- \cap C_c \,],
\]
together with the corresponding upper-bound coefficients
\[
\hat{\alpha}_i := \sum_c \alpha_{ic},
\qquad
\hat{\beta}_i := \sum_c \beta_{ic}.
\]
Binary indicators $y_i$ and $z_i$ are used to indicate whether image $i$ is included
in the final result $\mathcal{O}$.
We further impose a budget constraint that limits the number of selected clauses to at most $M$.

We formulate logic composition in the following MIP:
\begin{align}
\max_{u,y,z}\quad
& 2 \sum_{i \in \mathcal{I}^+} y_i
- \lambda \Bigl(|\mathcal{I}^+|
+ \sum_{i \in \mathcal{I}^+} y_i
+ \sum_{i \in \mathcal{I}^-} z_i \Bigr)
- \alpha \sum_c |P_c|\, u_c
\tag{3a}\label{eq:stage2_obj}\\
\text{s.t.}\quad
& y_i \le \sum_c \alpha_{ic} u_c \le \hat{\alpha}_i\, y_i,
\qquad \forall i \in \mathcal{I}^+,
\tag{3b}\label{eq:stage2_gt}\\
& z_i \le \sum_c \beta_{ic} u_c \le \hat{\beta}_i\, z_i,
\qquad \forall i \in \mathcal{I}^-,
\tag{3c}\label{eq:stage2_ngt}\\
& \sum_c u_c \le M,
\tag{3d}\label{eq:stage2_budget}\\
& u_c,\, y_i,\, z_i \in \{0,1\}.
\tag{3e}\label{eq:stage2_bin}
\end{align}

The F1 score corresponding to the selected clause union
$\mathcal{O} = \bigcup_{c:u_c=1} C_c$
is given by
\begin{equation}
F_1(\mathcal{O})
=
\frac{2\sum_{i \in \mathcal{I}^+} y_i}
{|\mathcal{I}^+| + \sum_{i \in \mathcal{I}^+} y_i + \sum_{i \in \mathcal{I}^-} z_i},
\label{eq:stage2_f1}
\end{equation}
and the parameter $\lambda$ in~\eqref{eq:stage2_obj}
is updated via Dinkelbach iteration as
\begin{equation}
\lambda
\leftarrow
\frac{2\sum_{i \in \mathcal{I}^+} y_i}
{|\mathcal{I}^+| + \sum_{i \in \mathcal{I}^+} y_i + \sum_{i \in \mathcal{I}^-} z_i}.
\label{eq:stage2_update}
\end{equation}

The objective in~\eqref{eq:stage2_obj} is obtained via the Dinkelbach reformulation
of the fractional F1 objective in~\eqref{eq:stage2_f1}.
At each iteration, an MIP is solved to maximize the difference between the numerator
and the denominator scaled by the current value of $\lambda$, and $\lambda$ is
updated to the resulting F1 score until convergence.
The complexity penalty $\alpha \sum_c |P_c|\, u_c$ discourages selecting long clauses.
Constraints~\eqref{eq:stage2_gt}--\eqref{eq:stage2_ngt} link clause selection to
image-level inclusion through $y_i$ and $z_i$, and
Constraint~\eqref{eq:stage2_budget} limits the number of selected clauses.

In practice, allowing unions over many enumerated clauses often produces bloated
and unintuitive DNF expressions that are difficult to interpret or learn.
For example, the resulting composition may take the form
\[
(A \setminus B) \;\cup\; (A \cap C) \;\cup\; D,
\]
where the underlying retrieval logic is fragmented across multiple overlapping
clauses.
To facilitate effective learning of tool-call trajectories by the planner,
we therefore adopt the less flexible but more structured Stage~II formulation
in Section~\ref{sec:method:stage2} in the main method.

\section{Expressive Completeness Proof}

The predefined two-clause composition (Appendix~\ref{appendix:mip_details}) in OSCAR is a practical choice for \textbf{stability and MIP tractability}, rather than a fundamental limit of expressiveness. More generally, atomic retrieval outputs generate a Boolean set algebra, and \textbf{any retrieval target} composed from them via union, intersection, and difference can be rewritten in disjunctive normal form (DNF). Hence, when extended to more complex queries, OSCAR can be naturally generalized from the current two-clause composition to a richer multi-clause DNF program under the same optimization-steered framework. We provide proof of DNF Completeness of Atomic Retrieval Composition below.

\begin{theorem}[DNF Completeness of Atomic Retrieval Composition]

Let $U$ denote the candidate universe, and let $\{S_r\}_{r\in\mathcal R}$ be atomic retrieval sets with $\cup,\cap,\setminus$. For any $E \in \mathfrak A$, there exists
\begin{equation}
    E = \bigcup_{j=1}^{m}\ \bigcap_{r\in C_j} L_r,
\quad
L_r \in \{S_r,\ U\setminus S_r\}.
\end{equation}

\end{theorem}

\begin{proof}

For each $u \in U$, define $X_r(u) = \mathbf{1}[u \in S_r]$.  
Any $E$ induces a Boolean formula $\varphi(\{X_r(u)\})$ with:
\begin{equation}
    \cup \leftrightarrow \lor,\quad
\cap \leftrightarrow \land,\quad
A\setminus B \leftrightarrow A \land \neg B.
\end{equation}

By the DNF theorem, $\varphi$ can be written as
\begin{equation}
    \varphi \equiv \bigvee_{j=1}^{m} \bigwedge_{r\in C_j} \ell_{jr},
\quad
\ell_{jr} \in \{X_r, \neg X_r\}.
\end{equation}

Mapping back:
\begin{equation}
    X_r \mapsto S_r,\quad \neg X_r \mapsto U\setminus S_r,
\end{equation}

gives
\begin{equation}
    E = \bigcup_{j=1}^{m}\ \bigcap_{r\in C_j} L_r.
\end{equation}
This establishes the expressive completeness of atomic retrieval composition, implying that OSCAR can in principle generalize to more complex queries beyond the current two-clause instantiation.

\end{proof}

\section{Binary Logit-Based Relevance Scoring}
\label{appendix:yesno_scoring}
Our verifier computes a relevance score by comparing the likelihood of generating \texttt{``yes''} and \texttt{``no''} as the \emph{next decoded token}. 
Let \(\mathcal{V}\) denote the vocabulary, and let \(z_t\) be the model's output logit for token \(t \in \mathcal{V}\) under a fixed context which consists of a reference image, a modification query, and a candidate image. The corresponding prompt for the verifier is provided in Appendix~\ref{appendix:prompts}.

\textbf{Step 1: Next-token distribution over the whole vocabulary.}
By definition, VLM induces a normalized probability distribution over the next token via softmax:
\begin{equation}
P(t) \;=\; \frac{\exp(z_t)}{\sum\limits_{v \in \mathcal{V}} \exp(z_v)} , \quad t \in \mathcal{V}.
\label{eq:full_softmax}
\end{equation}
In particular,
\begin{equation}
P(\texttt{yes}) \;=\; \frac{\exp(z_{\texttt{yes}})}{\sum\limits_{v \in \mathcal{V}} \exp(z_v)}, 
\qquad
P(\texttt{no}) \;=\; \frac{\exp(z_{\texttt{no}})}{\sum\limits_{v \in \mathcal{V}} \exp(z_v)}.
\label{eq:yes_no_full}
\end{equation}

\textbf{Step 2: Restricting to a binary decision and renormalization.}
We interpret relevance as a binary decision within the subset \(\{\texttt{yes}, \texttt{no}\}\).

\begin{align}
P(\texttt{yes}\mid \texttt{yes}\ \text{or}\ \texttt{no})
&=
\frac{P(\texttt{yes})}{P(\texttt{yes}) + P(\texttt{no})} \\
&=
\frac{\frac{\exp(z_{\texttt{yes}})}{\sum_{v \in \mathcal{V}} \exp(z_v)}}
{\frac{\exp(z_{\texttt{yes}})}{\sum_{v \in \mathcal{V}} \exp(z_v)} +
 \frac{\exp(z_{\texttt{no}})}{\sum_{v \in \mathcal{V}} \exp(z_v)}} \\
&=
\frac{\exp(z_{\texttt{yes}})}
{\exp(z_{\texttt{yes}}) + \exp(z_{\texttt{no}})} \\
&=
\frac{1}{1 + \exp\!\left(z_{\texttt{no}} - z_{\texttt{yes}}\right)} \\
&=
\sigma\!\left(z_{\texttt{yes}} - z_{\texttt{no}}\right).
\label{eq:binary_sigmoid}
\end{align}
where \(\sigma(x) = \frac{1}{1+\exp(-x)}\) is the sigmoid function.
In our implementation, we use this binary probability as the relevance score \(s_u=\sigma(z_{yes}-z_{no})\).

\section{Prompts}
\label{appendix:prompts}
For completeness, we present the planner, image captioning, and verifier prompts used in our OSCAR framework below.

\onecolumn
\tcbset{
    promptstyle/.style={
        colback=backorange,
        colframe=frameorange,
        fonttitle=\bfseries,
        arc=2pt,
        breakable, 
        left=2mm, right=2mm, top=2mm, bottom=2mm,
        boxrule=0.5pt,
        fontupper=\small,
        fontlower=\small,
    }
}

\lstdefinelanguage{json}{
    basicstyle=\ttfamily\small,
    numbers=none,
    showstringspaces=false,
    breaklines=true,
    frame=single,
    backgroundcolor=\color{white},
    stringstyle=\color{red},
    keywordstyle=\color{blue},
    commentstyle=\color{gray},
    morestring=[b]",
    morestring=[d]',
    linewidth=\textwidth
}

\begin{tcolorbox}[promptstyle, title=Planner Agent Prompt, width=\linewidth]
\small
\textbf{Role:} You are a search planner for Composed Image Retrieval (CIR).

\textbf{Task:} Given a reference image and a modification query describing ``Changes'' and ``Shared concept'', find target images that match the modified concept.

\vspace{\baselineskip}

\textbf{Available Models:}
\begin{itemize}[nosep, leftmargin=*]
    \item Embedding models: \textit{rzenembed}, \textit{ops\_embedding}, \textit{QQMM\_embed\_v2}
    \item Caption models: \textit{bge\_m3}, \textit{Qwen3\_Embed\_4B}, \textit{Qwen3\_Embed\_8B}
\end{itemize}

\vspace{\baselineskip}

\textbf{Available Tools:}
\begin{itemize}[nosep, leftmargin=*]
    \item \textit{embedding\_search\_\{model\_name\}(query, top\_k, reference\_image)}: Visual similarity search. Best for visual attributes, colors, spatial relationships.
    \item \textit{caption\_search\_\{model\_name\}(query, top\_k, reference\_image)}: Semantic text search on captions. Best for semantic descriptions, conceptual changes.
    \item \textit{set\_operation(steps)}: Combine results (INTERSECT, UNION, DIFFERENCE). \textbf{MUST be the last tool call when using 2+ tools.}
\end{itemize}

\par\noindent\rule{\textwidth}{0.4pt}

\textbf{CRITICAL - Set Operation Rules:}
\begin{itemize}[nosep, leftmargin=*]
    \item \textbf{INTERSECT (AND):} Return items in ALL operands. Use when results must satisfy multiple criteria.
    \item \textbf{UNION (OR):} Return items in ANY operand. Use for maximum coverage.
    \item \textbf{DIFFERENCE (A - B):} Return items in first but NOT in second. Use to exclude specific unwanted features.
\end{itemize}

\par\noindent\rule{\textwidth}{0.4pt}

\textbf{Strategy - How to Handle CIR Queries:}

\textbf{Step 1: Analyze Additions, Removals, or Changes}
\begin{itemize}[nosep, leftmargin=*]
    \item \textbf{When to use DIFFERENCE:} Only when an object is COMPLETELY ABSENT (\eg, ``has no X'', ``without X'').
    \item \textbf{When NOT to use DIFFERENCE:} If the object exists but changed attributes (\eg, ``has fewer people'', ``is in greyscale'', ``has different color'').
    \item \textbf{Negative Query Rule:} Negative queries must ONLY contain the absent object. NEVER include shared concepts (\eg, if shared is ``cake'', negative query should be ``toppings'', NOT ``cake with toppings'').
\end{itemize}

\vspace{\baselineskip}

\textbf{Step 2: Generate Tool Calls}
\begin{itemize}[nosep, leftmargin=*]
    \item \textbf{Positive tools:} Search for ``Shared concept'' + additions.
    \item \textbf{Negative tools:} ONLY search for objects that DON'T EXIST.
\end{itemize}
\vspace{\baselineskip}

\textbf{Step 3: Set Operation (Mandatory)}
\begin{itemize}[nosep, leftmargin=*]
    \item Union all positive tools $\rightarrow$ ``positive''.
    \item Intersect all negative tools $\rightarrow$ ``negative''.
    \item \textit{DIFFERENCE(positive, negative)} $\rightarrow$ ``final''.
\end{itemize}

\par\noindent\rule{\textwidth}{0.4pt}

\textbf{Static Reference Examples (Learn the format!)}

\par\noindent\rule{\textwidth}{0.4pt}

\textbf{Below are examples showing the correct tool call FORMAT (not related to your query):}
\begin{lstlisting}[language=json]
Query: Changes: shows two people and has a more colorful background. Shared concept: a girl with a traditional Chinese umbrella.
Reference image: 000000271520.jpg
Reference caption: A young woman in traditional Chinese attire performs gracefully on a dimly lit stage, holding a large, ornate red parasol.

Analysis:
	- Shared: girl with Chinese umbrella/parasol, traditional attire
	- Changes: ADD two people, ADD colorful background (instead of dimly lit stage)
    - No removals, No DIFFERENCE needed, just UNION
	- Search strategy: Find images with girl + Chinese umbrella + multiple people + colorful setting


<tool_call>
{"name": "embedding_search_QQMM_embed_v2", "arguments": {"query": "Two people with traditional Chinese umbrella parasol, colorful vibrant background", "top_k": 20, "reference_image": "000000271520.jpg"}}
</tool_call>
<tool_call>
{"name": "embedding_search_rzenembed", "arguments": {"query": "Girl with Chinese parasol umbrella, two people together, colorful setting", "top_k": 15, "reference_image": "000000271520.jpg"}}
</tool_call>
<tool_call>
{"name": "embedding_search_ops_embedding", "arguments": {"query": "Chinese umbrella, traditional dress, multiple people, colorful background", "top_k": 15, "reference_image": "000000271520.jpg"}}
</tool_call>
<tool_call>
{"name": "caption_search_Qwen3_Embedding_8B", "arguments": {"query": "girl with Chinese umbrella parasol, two people, colorful vibrant background", "top_k": 10, "reference_image": "000000271520.jpg"}}
</tool_call>
<tool_call>
{"name": "set_operation", "arguments": {"steps": [  {"name": "final", "op": "UNION", "operands": [0, 1, 2, 3]}]}}
</tool_call>
\end{lstlisting}

\par\noindent\rule{\textwidth}{0.4pt}
\textbf{Example2 - DIFFERENCE (Remove elements)}

\begin{lstlisting}[language=json]
Query: Changes: has no toppings and shows the rest in the background. Shared concept: a slice of cake on a plate in the foreground.
Reference image: 000000119289.jpg
Reference caption: A slice of chocolate cake with white frosting and cherry toppings sits on a white plate.
Analysis:
	- Shared: slice of cake on plate, foreground
    - Changes: REMOVE toppings (no cherry, no frosting decorations), ADD background elements
	- has no toppings, toppings are ABSENT: Use DIFFERENCE
	-  CRITICAL: cake EXISTS in target! NEVER include cake in negative query!
    Search strategy: Find cake slices WITHOUT toppings/decorations, with background visible.Use DIFFERENCE to exclude images with visible toppings.
    Warning: Negative query ONLY contains toppings/cherries/decorations, NOT cake!

<tool_call>
{"name": "embedding_search_QQMM_embed_v2", "arguments": {"query": "Plain slice of cake on plate, simple cake, background visible", "top_k": 20, "reference_image": "000000119289.jpg"}}
</tool_call>
<tool_call>
{"name": "embedding_search_rzenembed", "arguments": {"query": "Simple cake slice on plate in foreground, background scene", "top_k": 15, "reference_image": "000000119289.jpg"}}
</tool_call>
<tool_call>
{"name": "embedding_search_ops_embedding", "arguments": {"query": "Cake slice plate foreground, plain simple, background visible", "top_k": 25, "reference_image": "000000119289.jpg"}}
</tool_call>
<tool_call>
{"name": "caption_search_Qwen3_Embedding_8B", "arguments": {"query": "plain cake slice on plate, background visible", "top_k": 10, "reference_image": "000000119289.jpg"}}
</tool_call>
<tool_call>
{"name": "embedding_search_QQMM_embed_v2", "arguments": {"query": "cherry topping, frosting decorations, fruit toppings, garnish", "top_k": 20, "reference_image": "000000119289.jpg"}}
</tool_call>
<tool_call>
{"name": "embedding_search_rzenembed", "arguments": {"query": "cherries, whipped cream topping, dessert decorations, garnish", "top_k": 20, "reference_image": "000000119289.jpg"}}
</tool_call>
<tool_call>
{"name": "caption_search_Qwen3_Embedding_8B", "arguments": {"query": "toppings, cherries, frosting decorations, garnished dessert", "top_k": 10, "reference_image": "000000119289.jpg"}}
</tool_call>
<tool_call>
{"name": "set_operation", "arguments": {"steps": [{"name": "positive", "op": "UNION", "operands": [0, 1, 2, 3]},  {"name": "negative", "op": "INTERSECT", "operands": [4, 5, 6]}, {"name": "final", "op": "DIFFERENCE", "operands": ["positive", "negative"]}]}}
</tool_call>
\end{lstlisting}

\par\noindent\rule{\textwidth}{0.4pt}

\textbf{RETRIEVED SIMILAR CASES}

Below are some similar cases retrieved based on your query and reference image.
You can learn these cases patterns for: \begin{itemize}[nosep, leftmargin=*]
    \item Query Rewrite
    \item Tool selection and composition
    \item  top\_k values
\end{itemize} 

\textcolor{blue}{\{retrieved\_cases\}}

\par\noindent\rule{\textwidth}{0.4pt}

Now solve the following task:
Modification Query: \textcolor{blue}{\{query\}}

Reference Image: \textcolor{blue}{\{ref\_img\}}

Reference Image Caption: \textcolor{blue}{\{ref\_caption\}}

\vspace{\baselineskip}
Now generate tool\_calls to retrieve target images based on the provided Modification Query and Reference Image.
\end{tcolorbox}

\begin{tcolorbox}[promptstyle, title=Image Captioning Prompt, width=\linewidth]

Describe this image comprehensively:

1. Main subjects and their attributes

2. Actions, interactions, and relationships

3. Scene, background, and context

4. Any visible text (signs, labels, watermarks, etc.)

Return your response as JSON:

\texttt{\textasciigrave\textasciigrave\textasciigrave}json

\{

    "caption": "your comprehensive description here"

\}

\texttt{\textasciigrave\textasciigrave\textasciigrave}
\end{tcolorbox}

\begin{tcolorbox}[promptstyle, title=Verifier Agent Prompt, width=\linewidth]

\textcolor{blue}{\{Reference image\}}

\textcolor{blue}{\{Candidate image\}}

Modification query: \textcolor{blue}{\{<query>\}}

\medskip
\textbf{Instruction:}

You are given two images.  

The first image is the reference image.  

The second image is a candidate image.

Given the modification query: \textcolor{blue}{\{<query>\}}, determine whether the candidate image matches what you would expect after applying the modification to the reference image.

Answer with \textbf{``yes''} or \textbf{``no''} only.
    
\end{tcolorbox}

\end{document}